\documentclass[journal]{IEEEtran}

\ifCLASSINFOpdf
\usepackage[square,sort,comma,numbers]{natbib}
\usepackage{amsmath}
\usepackage{times}
\usepackage{graphicx}
\usepackage{color}
\usepackage{multirow}
\usepackage{rotating}
\usepackage{bbm}
\usepackage{latexsym}

\usepackage{varwidth,xcolor}
\usepackage{tabularx}
\usepackage{amsthm}
\usepackage{amssymb}
\usepackage{graphicx}
\usepackage{algorithm}
\usepackage{algorithmicx}
\usepackage{algpseudocode}
\usepackage{hyperref}
\usepackage{url}
\usepackage{times}
\usepackage{multirow}
\usepackage{graphicx}
\usepackage{subfigure}
\usepackage{amsmath}
\usepackage{amsfonts}
\usepackage{amssymb}
\usepackage{times}
\usepackage{bm}
\usepackage{color}
\usepackage{amsmath}
\usepackage{amsfonts}
\usepackage{amssymb}
\usepackage{stmaryrd}
\usepackage{multirow}
\usepackage{soul}
\usepackage{rotating}
\usepackage{pdflscape}

\def\bx{{\mathbf x}}

\def\E{{\mathbb E}}

\def\etal{{\em et al.\/}\,}

\def\0{{\bf 0}}

\def\E{{\mathbb E}}
\newtheorem{remark}{Remark}
\newtheorem{deftn}{Definition}
\newtheorem{theorem}{Theorem}
\newtheorem{proposition}[theorem]{Proposition}
\newtheorem{lemma}[theorem]{Lemma}

\begin{document}
%\hspace{13.9cm}1
%
%\ \vspace{20mm}\\
%
%{\LARGE $N$-ary Coding Schemes for Classification with Many Classes}
%
%\ \\
%{\bf \large Joey Tianyi Zhou$^{\displaystyle 1}$, Ivor~W.~Tsang$^{\displaystyle 2}$, ~Shen-Shyang~Ho$^{\displaystyle 1}$, and Klaus-Robert M$\ddot{\textbf{u}}$ller$^{\displaystyle 3, \displaystyle 4}$    }\\
%{$^{\displaystyle 1}$Institute of High Performance Computing, A*STAR, Singapore.}\\
%{$^{\displaystyle 2}$University of Technology, Sydney, Australia.}\\
%{$^{\displaystyle 3}$Machine Learning Group, Berlin Institute of Technology, Berlin, Germany
%.}\\
%{$^{\displaystyle 4}$Department of Brain and Cognitive Engineering, Korea
%University, Seoul, Korea.}
%%\ \\[-2mm]
%{\bf Keywords:} coding, multi-class classification
%
%\thispagestyle{empty}
%\markboth{}{NC instructions}
%%
%\ \vspace{-0mm}\\
%
%Abstract
\title{$N$-ary Error Correcting Coding Scheme}
\author{Joey~Tianyi~Zhou,~%\IEEEmembership{Member,~IEEE,}
        Ivor~W.~Tsang,~%\IEEEmembership{Fellow,~OSA,}
        Shen-Shyang~Ho~%\IEEEmembership{Fellow,~OSA,}
        and~Klaus-Robert~M$\ddot{\textrm{u}}$ller,~%\IEEEmembership{Life~Fellow,~IEEE}% <-this % stops a space
\thanks{J.~T. Zhou is with the Institute of High Performance Computing, Singapore,}% <-this % stops a space
\thanks{I.~W. Tsang is with the University of Sydney, Australia,}
\thanks{S.-S. Ho is with Nanyang Technological University, Singapore,}
\thanks{K.-R.~M$\ddot{\textrm{u}}$ller is with the Department of Machine Learning,
Berlin Institute of Technology, Berlin, Germany, and also with the Department
of Brain and Cognitive Engineering, Korea University, Seoul, Korea.}%% <-this % stops a space
}

\markboth{Submitted to Transaction on Information Theory }%
{Shell \MakeLowercase{\textit{et al.}}: Bare Demo of IEEEtran.cls for IEEE Journals}

\maketitle
\begin{abstract}
The coding matrix design plays a fundamental role in the prediction performance of the error correcting output codes (ECOC)-based multi-class task. {In many-class classification problems, e.g., fine-grained categorization, it is difficult to distinguish subtle between-class differences under existing coding schemes due to a limited choices of coding values.}   In this paper, we investigate whether one can relax existing binary and ternary code design to $N$-ary code design to achieve better classification performance. {In particular, we present a novel $N$-ary coding scheme that decomposes the original multi-class problem into simpler multi-class subproblems, which is similar to applying a divide-and-conquer method.} The two main advantages of such a coding scheme are as follows: (i) the ability to construct more discriminative codes and (ii) the flexibility for the user to select the best $N$ for ECOC-based classification. We show empirically that the optimal $N$ (based on classification performance) lies in $[3, 10]$ with some trade-off in computational cost. Moreover, we provide theoretical insights on the dependency of the generalization error bound of an $N$-ary ECOC on the average base classifier generalization error and the minimum distance between any two codes constructed. Extensive experimental results on  benchmark multi-class datasets show that the proposed coding scheme achieves superior prediction performance over the state-of-the-art coding methods.
\end{abstract}
%%%%%%%%%%%

\section{Introduction}
Many real-world problems are multi-class in nature.
To handle multi-class problems,  many  approaches have been proposed. One research direction focuses on solving  multi-class
problems directly. These approaches include  decision tree based methods
\cite{Quinlan:1986,CPT_UAI2009,Su:2006:FDT:1597538.1597619,BengioWG_NIPS2010,conf/uai/YangT11,GaoK_relaxed_hierarchy_ICCV2011,DengSBL_NIPS2011}. In particular,
decision-tree based algorithms label each leaf of the decision tree
with one of the $N_C$ classes, and internal nodes can be selected to
discriminate between these classes. The performance of decision-tree based algorithms heavily depends on the internal tree structure. Thus, these methods are usually vulnerable
to outliers. To achieve better generalization, \cite{conf/uai/YangT11,GaoK_relaxed_hierarchy_ICCV2011} propose to learn  the decision tree structure based on the large margin criterion. However,
these algorithms usually involve solving sophisticated optimization problems and their training time increases dramatically with the increase of the number of classes.  Contrary to these complicated methods,  K-Nearest Neighbour (KNN) \cite{1053964} is a simple but effective and stable approach to handle multi-class problems. However, KNN is sensitive to noise features and can therefore suffer from the curse-of-dimensionality.   Meanwhile,
Crammer \etal \cite{Crammer:2002:AIM:944790.944813,10.1371/journal.pone.0042947} propose a direct approach for
learning multi-class support vector machines (M-SVM) by deriving the
generalized notion of margins as well as separating hyperplanes.

Another research direction focuses on the error correcting output codes (ECOC) framework that decomposes a multi-class problem into multiple binary problems so that one reuses the well-studied binary classification algorithms for their simplicity and efficiency.
Many ECOC approaches~\cite{Dietterich91error-correctingoutput,liu2013learning,Rocha-TNNLS-2014,conf/cvpr/ZhaoX13,montazer2012error,ubeyli2007675,Garcia-Pedrajas:2011} have been proposed in recent years to design a good coding matrix.
They fall into  two categories: problem-independent and problem-dependent.
The challenge with problem-independent codings, such as  random
 ECOC \cite{4668347}, is that they are not designed and optimized for a particular dataset. In fact, there is little guarantee that the created base codes are always discriminative for the multi-class classification task. Therefore, they usually require a large number of base classifiers generated by the
pre-designed coding matrix \cite{Allwein:2001:RMB:944733.944737}.
To overcome this weakness, problem-dependent methods such as discriminant
ECOC (DECOC) \cite{Pujol06discriminantecoc:} and node embedding ECOC (ECOCONE) \cite{conf/icpr/EscaleraP06} are proposed.
Recently, subspace approaches such as  subspace ECOC
\cite{Bagheri2013176} and adaptive ECOC \cite{ZhongIJCAI13} are proposed to further improve the ECOC classification framework. Though all the above-mentioned variations of the ECOC approach endeavor to enhance the ECOC paradigm for classification tasks, their designs are confined to binary $\{-1,1\}$ and ternary codes $\{-1,0,1\}$.  Such a code design constraint poses limitations on the error correcting capability of ECOC that relies on the minimum  distance, $\Delta_{\min}(M)$, between any distinct pair of rows in the coding matrix $M$. A larger $\Delta_{\min}(M)$ is more likely to rectify the errors committed by individual base classifiers \cite{Allwein:2001:RMB:944733.944737}.

{However, in the more challenging real-world applications, there exists multi-class problems where some of the classes are very similar and difficult to differentiate with each other. For example, in the fine-grained classification \cite{yao2011combining} unlike basic-level recognition, even humans
might have difficulty with some of the fine-grained categorization. One major challenge in fine-grained image classification
is to distinguish subtle between-class differences while
each class often has large within-class variation in the image
level \cite{wang2014object}. The existing binary ECOC codes cannot handle this challenge due to limited choices of coding values. It is highly possible that some classes out of multi-class classification problems are assigned with same or similar codes.  } To address this issue,
we investigate whether one can extend the existing binary or ternary coding scheme to  an \emph{$N$-ary coding} scheme to (i) allow users the flexibility to choose $N$ to construct the codes in order to (ii) improve the ECOC classification performance for a given dataset. The main contributions of this paper are as follows.
\begin{itemize}
\item We propose a novel $N$-ary coding scheme that achieves a large expected distance between any pair of rows in $M$ at a reasonable $N (> 3)$ for a multi-class problem (see Section \ref{sec:necoc}).
{The main idea of our coding scheme is to decompose the original multi-class problem into a series of smaller multi-class subproblems instead of binary classification problems. Suppose that a metaclass is a subset of classes. Also, suppose that all classes are in a one large
metaclass. So, in each level, there is a classifier to divide a metaclass into two smaller metaclasses.  This coding scheme is like a divide-and-conquer method.}
The two main advantages of such a coding scheme are as follows: (i) the ability to construct more discriminative codes and (ii) the flexibility for the user to select the best $N$ for ECOC-based classification.

\item  We provide theoretical insights on the dependency of the generalization error bound of a $N$-ary ECOC on the average base classifier generalization error and the minimum distance between any two constructed codes (see Section \ref{sec:Bound}). Furthermore, we conduct a series of empirical analyses to verify the validity of the theorem on the ECOC error bound (see section \ref{sec:experiments}).

\item We show empirically that the optimal $N$ (based on classification performance) lies in $[3, 10]$ with a slight trade-off in computational cost (see Section \ref{sec:experiments}).

\item We show empiricially the superiority of the proposed coding scheme  over the state-of-the-art coding methods for multi-class prediction tasks on a set of benchmark datasets (see Section \ref{sec:experiments}).

\end{itemize}
To the best of our knowledge, there is no previous work that attempts to extend and generalize the coding scheme to $N$-ary codes with $N> 3$.

The rest of this paper is organized as follows. Section \ref{sec:related_work} reviews the related work. Section \ref{sec:necoc} presents the generalization from binary coding to $N$-ary coding. In Section \ref{sec:complexity},we give the complexity analysis of $N$-ary coding and compare it with other coding schemes with the SVM classifier as a showcase.   Section \ref{sec:Bound} gives the error bound analysis of $N$-ary coding. Finally, Section \ref{sec:experiments} discusses our empirical studies and Section \ref{sec:conclude} concludes this work.

\begin{figure*}[htp!]
%\vspace{-3.5 cm}
\centering
\subfigure[  Multi-class Data] {\label{fig:intro:data}\includegraphics[trim = 6.5cm 5.5cm
6.1cm 5.5cm, clip, width=0.25\textwidth]{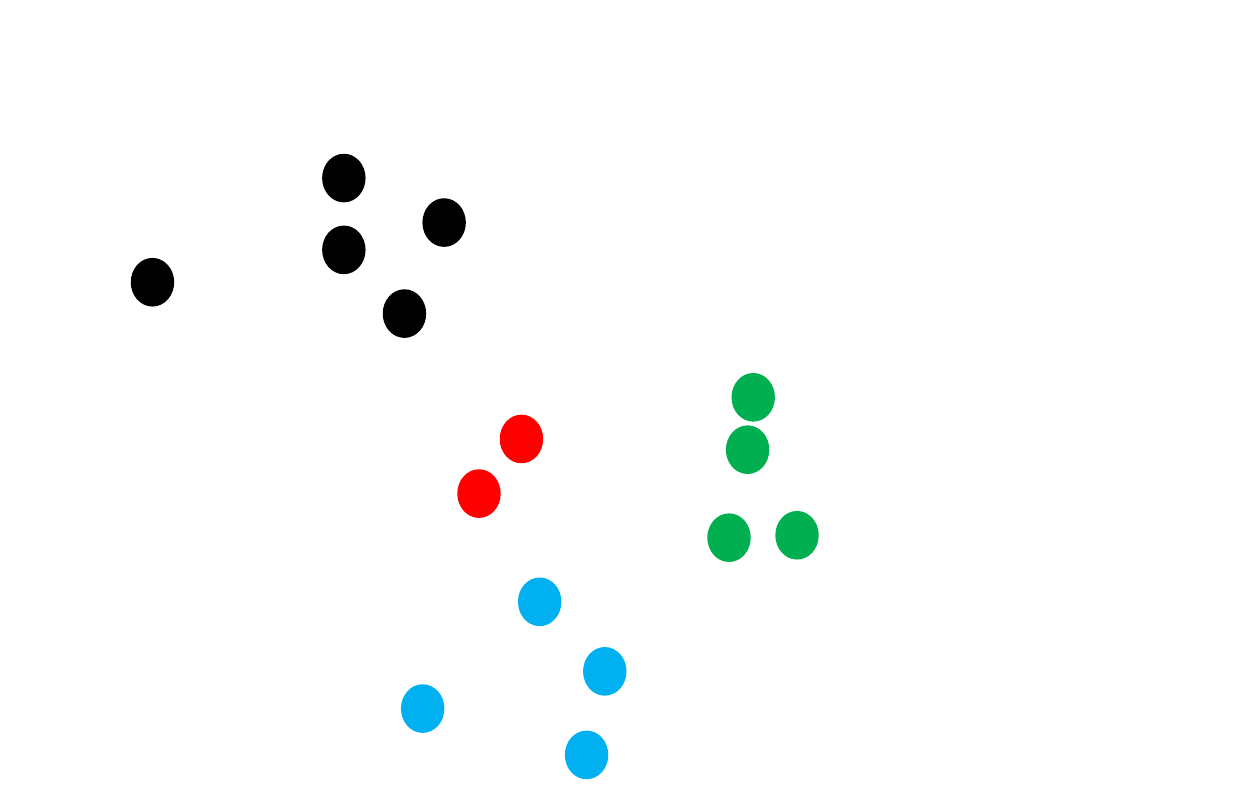}}%\vspace{-1.5cm}
\subfigure[ Binary code (e.g., OVA)] {\label{fig:intro:binary}\includegraphics[trim = 6.5cm 5.5cm
6.1cm 5.5cm, clip, width=0.25\textwidth]{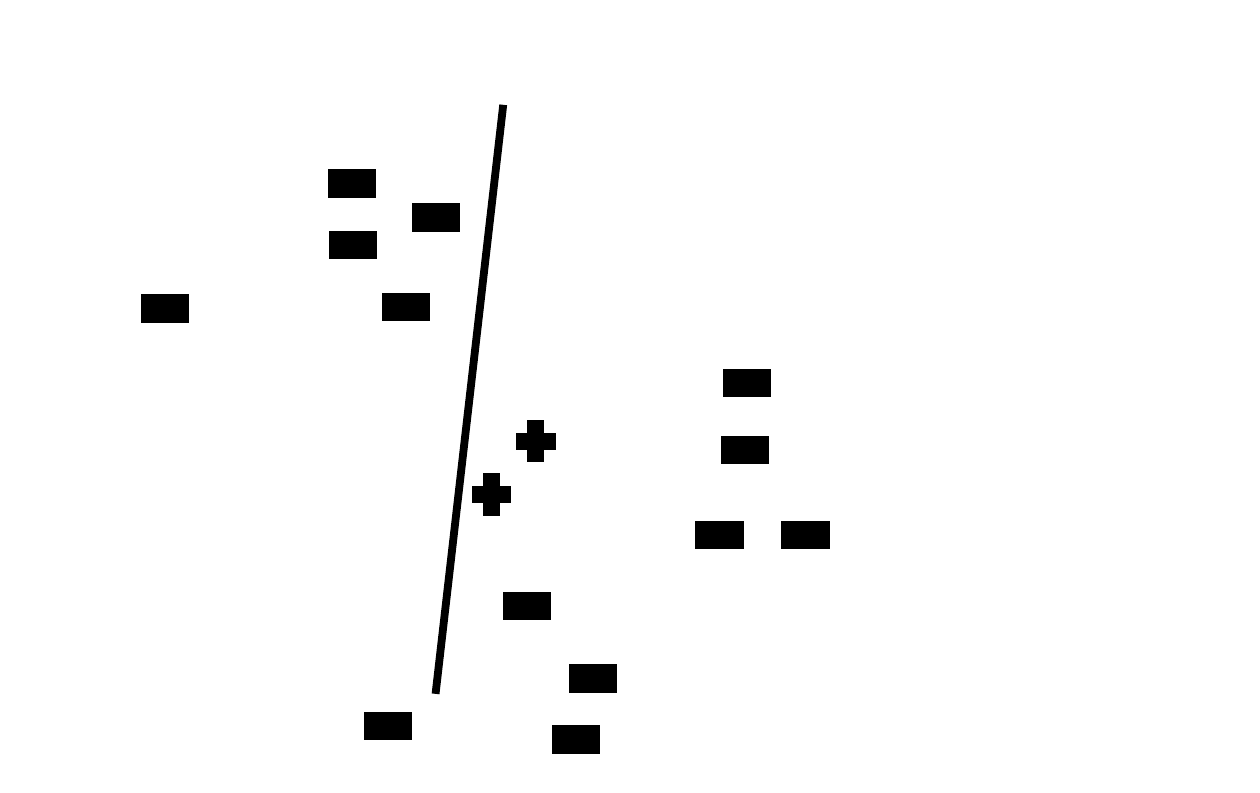}}%\vspace{-0.1 cm}
 \subfigure[ Ternary code (e.g., OVO)]{\label{fig:intro:ternary}\includegraphics[trim = 6.5cm 5.5cm
6.1cm 5.5cm, clip, width=0.25\textwidth]{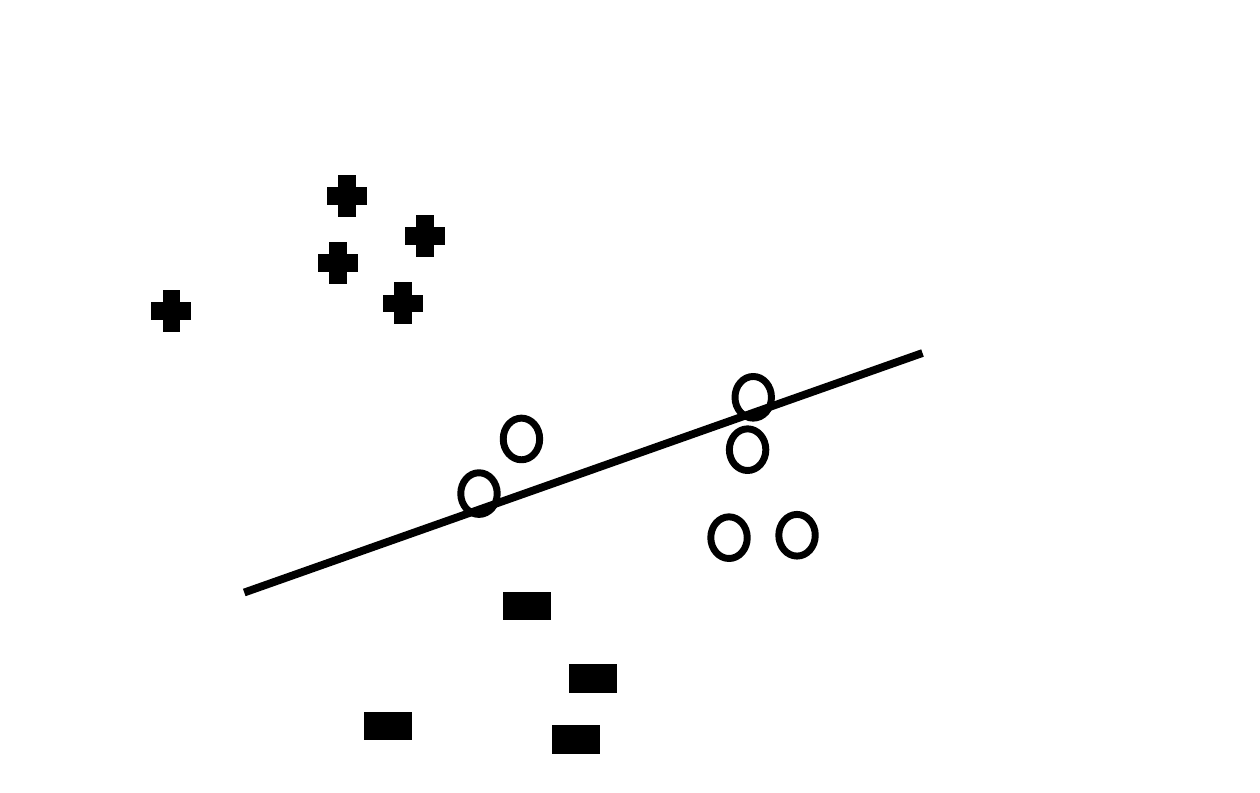}}%\vspace{-0.4 cm}
 \subfigure[$N$-ary Coding]{\label{fig:intro:nary}\includegraphics[trim = 6.5cm 5.5cm
6.1cm 5.5cm, clip, width=0.25\textwidth]{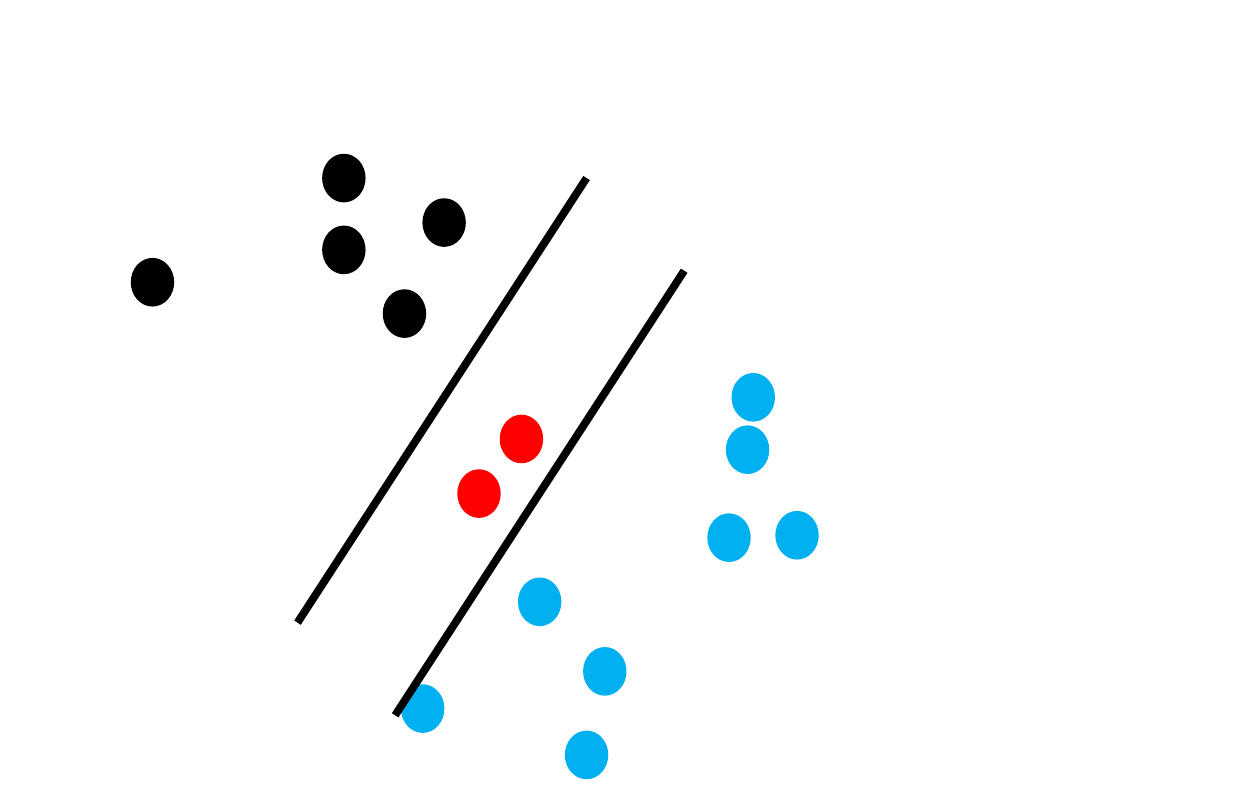}}
\caption{Fail cases of existing coding scheme: The binary coding scheme sometimes creates non-separate binary classification problems as shown in Figure \ref{fig:intro:binary}.
 The ternary coding scheme sometimes  creates cases where the data from the same class is assigned to different classes as shown in Figure \ref{fig:intro:binary}. $N$-ary coding scheme decomposes the
 difficult task into some smaller and easier tasks.   } \label{fig:intro}
\end{figure*}

\section{Related Work}\label{sec:related_work}
Many ECOC approaches~\cite{Allwein:2001:RMB:944733.944737,Pujol06discriminantecoc:,4668347,Bagheri2013176} have been proposed to design a good coding matrix in recent years.
Most of them  fall into the following two categories.
The first one is problem-independent coding, such as OVO, OVA, random
dense ECOC, and random sparse ECOC \cite{4668347}.  However, the coding matrix design is not optimized for the training dataset or the instance labels. Therefore, these
approaches usually require a large number of base classifiers generated by the
pre-designed coding matrix.
For example, the random dense
ECOC coding approach aims to construct the ECOC matrix $M \in
\{-1,1\}^{N_C\times N_L}$ where $N_C$ is the number of classes, $N_L$ is the code length,
and its elements are randomly chosen as either 1 or -1 \cite{Dietterich95solvingmulticlass}.
\cite{Allwein:2001:RMB:944733.944737} extends this binary
coding scheme to ternary coding by using a coding
matrix $M \in \{-1,0,1\}^{N_C\times N_L}$ where the classes
corresponding to 0 are not considered in the learning process.
%However, when $N_C$ is very large (\emph{e.g.}  hundreds or  thousands), OVO or even OVA also become intractable.
 Allwein \etal \cite{Allwein:2001:RMB:944733.944737} suggest that dense and sparse random ECOC approaches require only $10\log_2(N_C)$ and $15\log_2(N_C)$ base
classifiers, respectively,  to achieve optimal results.
However, a random ECOC coding approach cannot  guarantee that the created base codes are always discriminative for the multi-class classification task.
Therefore, it is possible that either  some base classifiers that are redundant   for the prediction exist or  badly designed base classifiers are constructed.

To overcome this problem, some problem-dependent methods have been
proposed. In particular, the coding matrix is learned by taking the instances
as well as labels into consideration. For instance, discriminant
ECOC (DECOC) \cite{Pujol06discriminantecoc:} embeds a binary
decision tree into the ternary codes. Its key idea is to find the most
discriminative hierarchical partition of the classes which
maximizes the quadratic mutual information between the data subsets
and the class labels created for each subset. As a result, DECOC
needs  exactly $N_C-1$ base classifiers which significantly
accelerate the testing process  without
sacrificing  performance.  %Following this idea, \cite{conf/uai/YangT11} further propose to find the most discriminative tree for prediction in terms of maximum separating margin.
However,  this decision tree based method has one major drawback: if the
parent node misclassifies an instance, the mistake will be
propagated to all the subsequent child nodes.

To address this weakness, Escalera \etal \cite{conf/icpr/EscaleraP06} proposed to optimize  node embedding for ECOC, called ECOCONE. For this approach, one initializes
a problem-independent ECOC matrix (usually OVA) and
 iteratively adds the base classifiers that discriminate the most confusing pairs of
classes into the previous ECOC ensemble to improve  performance.
However, ECOCONE suffers from three major limitations. Firstly, its
performance relies on the initial coding matrix. If the initial
coding matrix fails to perform well, the final results of ECOCONE
are usually unsatisfactory. Secondly,  its improvement is usually hindered
if it fails to discriminate the most confusing pairs. Lastly, similar to
DECOC, the training process is also time-consuming.

In addition to the above problem-dependent ECOC methods, Rocha \etal \cite{Rocha-TNNLS-2014} considers the correlation and joint probability of base binary classifiers  to reduce the number of base classifiers without sacrificing accuracy in the ECOC code design.
More recently, Zhao \etal \cite{conf/cvpr/ZhaoX13} proposed to impose the sparsity criterion into output code learning. It is shown to have much better performance and scalability to large scale classification problems compared to traditional methods like OVA. However, it involves a series of  complex optimizations to solve the proposed model with integer constraints in learning the ECOC coding matrix.
%iteratively increase the code length based on the initial coding matrix to further improve the performance.
%another kind of problem-dependent approaches have been proposed, which are usually
%For example,

Different from the aforementioned methods, some
subspace approaches have been developed. For example,  subspace ECOC
\cite{Bagheri2013176}  is  based on using different feature subsets
for learning each base classifier to improve independence among
classifiers.  Adaptive ECOC \cite{ZhongIJCAI13} reformulates the
ECOC models into multi-task learning where the subspace
for data and base classifiers are learned.

Due to the favorable properties and promising performance of ECOC approaches for the classification task, they have been applied to real-world classification applications such as face verification \cite{journals/ivc/KittlerGWM03}, ECG beats classification \cite{ubeyli2007675}, and even beyond  multi-class problems, such as feature extraction \cite{DBLP:journals/pr/ZhongL13} and fast similarity search \cite{Yu:2010:EOH:1937728.1937730}.

Though all the above-mentioned variations of the ECOC approach endeavor to enhance the ECOC paradigm for the classification task, their designs are still based on either binary or ternary codes which lack some desirable properties available in their generalized form.

\section{From A Binary to $N$-ary Coding Matrix}\label{sec:necoc}\vspace{-1 mm}
In this section, we discuss necessities and advantages of $N$-ary coding scheme from aspects of \textbf{column correlation  of coding matrix} and \textbf{separation between codewords of different classes}.

Existing ECOC algorithms  constrain the coding values either in $\{-1,1\}$ or $\{-1,0,1\}$.
 A lot of studies show that when there are sufficient classifiers,  ECOC can reach stable and reasonable performance
  \cite{Rocha-TNNLS-2014,Dietterich95solvingmulticlass}. Nevertheless, binary and ternary codes can generate
   at most $2^{N_C}$ and $3^{N_C}$ binary classifiers, where $N_C$ denotes the number of classes.    On the other hand,  due to  limited choices of coding values, existing codes tend to create correlated and redundant classifiers and make them less effective ``voters".
    Moreover, some studies show that  binary and ternary codes usually require only $10\log_2(N_C)$ and $15\log_2(N_C)$ base
classifiers, respectively,  to achieve optimal results \cite{Allwein:2001:RMB:944733.944737,4668347}.  {Furthermore, when the original multi-class problem is difficult, the
existing coding schemes cannot handle well. For example, as shown in Figure \ref{fig:intro:binary}, the binary codes like OVA may create difficult base binary classification tasks. Ternary codes
may cause cases where the test data from the same class is assigned to different classes.}%\vspace{-3 mm}
\begin{table}[h]
%\vspace{-0.2cm}
\caption{\label{tab:showcase_ECOC_matrix}An example of $N$-ary coding matrix $M$ with $N=4$ and $N_L = 6$.}
 {\small%\vspace{-0.1mm}
\begin{tabular}{|c||c|c|c|c|c|c|}
  \hline
        &$M_1$& $M_2$&$M_3$&$M_4$&$M_5$&$M_6$\\\hline
  $C_1$ & 1  & 1  & 2  & 4  & 1  & 1  \\
  $C_2$ & 2  & 1  & 1  & 3  & 2  & 1  \\
  $C_3$ & 3  & 2  & 1  & 2  & 3  & 1  \\
  $C_4$ & 4  & 3  & 1  & 1  & 4  & 2  \\
  $C_5$ & 4  & 3  & 2  & 2  & 4  & 3  \\
  $C_6$ & 4  & 3  & 3  & 3  & 3  & 4  \\
  $C_7$ & 3  & 4  & 4  & 4  & 2  & 4  \\
  \hline
\end{tabular}}%\vspace{-1.5mm}
\centering %\vspace{-5mm}
\end{table}

To address these issues, we extend the binary or ternary codes to $N$-ary codes.
One example of the $N$-ary coding matrix to represent seven classes is shown in Table~\ref{tab:showcase_ECOC_matrix}. Unlike the existing ECOC framework, a row of coding matrix $M$ represents the code of each class and the code consists of $N_L$ numbers in  $\{1\cdots N\}$, where $N>3$;
while a column $M_s$ of $M$ represents the $N$ partitions of classes to be considered. To be specific, the $N$-ary ECOC approach consists of four main steps:
\begin{enumerate}
\item Generate an $N$-ary matrix $M$ by uniformly random sampling from a range $\{1.. N\}$ (e.g., Table~\ref{tab:showcase_ECOC_matrix}).
 \item For each of the $N_L$ matrix columns,  partition original training data into $N$ groups based on the new class assignments and build an $N$-class classifier.
  \item  Given a test example $\bx_t$, use the $N_L$ classifiers to output $N_L$ predicted labels for the testing output code (e.g., $f(\bx_t) = [4, 3, 1, 2, 4, 2]$).
  \item  Final label prediction $y_t$ for $\bx_t$ is the nearest class based on minimum  distance between the training and the testing output codes (e.g., $y_t = \arg\min_i d(f(\bx_t),C_i) = 4$ ).
\end{enumerate}

 One notes that $N$-ary ECOC randomly breaks a large multi-class problem into a number of smaller multi-class subproblems. These subproblems are more complicated than binary problems and they incur additional computational cost. Hence, there is a trade-off between error correcting capability and computational cost.\footnote{More complexity analyses can be found from Section \ref{sec:complexity}.}  Fortunately, our empirical studies  indicate that $N$ does not need to be too large to achieve good classification performance.
\subsection{Column Correlations   of Coding Matrix}\vspace{-1mm}
\begin{figure}[ht!]
%\vspace{-0.4 cm}
\centering
\includegraphics[trim = 0.5cm 3.5cm
1.1cm 6.5cm, clip,  width = 0.75\columnwidth %height = 3.2in, width=3.6in
]{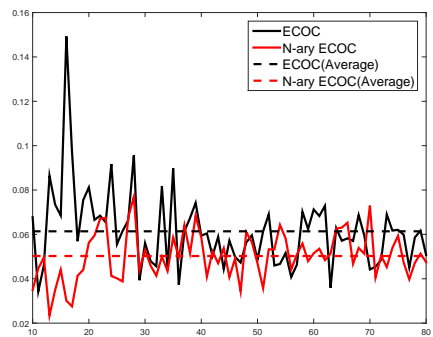}\vspace{-1 cm}%\hspace{0.5cm}
\caption{Column Correlation Comparison ( PCC v.s. $N_L$ )}\label{fig:column_corr}\vspace{-3 mm}
\end{figure}
In traditional ECOC, it suggests longer codes, i.e., $N_L$ is larger, however more binary base classifiers are likely
to be more correlated. Thus, more base classifiers created by binary or ternary codes are not effective
for final multi-class classification.
To illustrate the advantage of $N$-ary coding scheme in creating uncorrelated codes for base classifications, we conduct an experiment
 to investigate the column  correlations   of matrix $M$.  The results are shown in the Figure \ref{fig:column_corr}. In the experiment,
 we set $N_C = 20, N = 5,$ and $N_L$ varies in $[10,80]$, and use Pearson's correlation (PCC) which is a normalized correlation measure
 that eliminates the scaling effect of the codes. From  Figure \ref{fig:column_corr}, we observe that $N$-ary coding scheme achieves lower correlations for columns of coding matrix compared to conventional ternary ECOC.   Especially, when the number of tasks is small, the correlations over the created tasks for ECOC is higher than that of the $N$-ary ECOC. Therefore, an $N$-ary coding scheme not only provides  more flexibility in creating a coding matrix, but also generates codes  that are less correlated and less redundant, compared to traditional ECOC coding schemes.

\subsection{Separation Between Codewords of Different Classes}\vspace{-1 mm}
Apart from the column correlation, the row separation is another important measure to evaluate the error correcting ability of the coding matrix $M$ \cite{Dietterich95solvingmulticlass,Allwein:2001:RMB:944733.944737}.  The codes for different classes are expected to be as dissimilar as possible. If codes (rows) for different classes are similar, it is easier to commit errors.  Thus, the capability of error correction relies on the minimum distance, $\Delta_{\min}(M)$ or expected $\Delta(M)$  for any distinct pair of rows in the coding matrix $M \in
\{-1, 0, 1\}^{N_C\times N_L}$ where $N_C$ is the number of classes, and $N_L$ is the code length.   Both the absolute distance and the Hamming distance can
serve as the measure of row separation. The key difference between these two distances is that Hamming distance measures a scale-invariant difference. Specifically, the Hamming distance only cares about the number of different elements. It ignores the scale of the difference.

\noindent\textbf{Hamming Distance:}
One can use the \textit{{generalized Hamming distance}}  to calculate the $\Delta^{Ham}(M)$ for the existing coding schemes, which is defined as follows,
\begin{deftn}[Generalized Hamming Distance]\label{dfn:Ham} Let $M(r_1,:), M(r_2,:)$ denote row $r_1, r_2$ coding vectors in coding matrix $M$ with length $N_L$, respectively. Then the generalized hamming distance can be expressed as
\small
\begin{eqnarray}%\hspace{-2.2mm}
\!\!\!\!&&\Delta^{Ham}(M(r_1,:),M(r_2,:)) = \nonumber\\
\!\!\!\!&& \!\!\!\!\sum_{s=1}^{N_L} \left\{\begin{array}{cl}
                                                    \!\!\!\!\!\!\!\!\!0\!\! &\!\!\textrm{if}~ M(r_1,s)\!=\!M(r_2,s)\wedge M(r_1,s)\neq 0\wedge M(r_2,s)\neq 0 \\
                                                      \!\!\!\!\!\!\!\!\!1\!\!& \!\!\textrm{if}~ M(r_1,s)\!\neq\! M(r_2,s)\wedge M(r_1,s)\neq 0\wedge M(r_2,s)\neq 0\\
                                                        \! \!\!\!0.5& \!\!\textrm{if} ~ M(r_1,s)\!=\! 0 \vee  M(r_2,s)\!=\! 0.\\
                                                    \end{array} \right. \nonumber
\end{eqnarray}
\end{deftn}
For the OVA coding, every two rows have exactly two
entries with opposite signs,  $ \Delta_{min}^{Ham(OVA)}(M)=2$. For the OVO coding, every two rows
have exactly one entry with opposite signs, $\Delta_{min}^{Ham(OVO)}(M)=\left(\left(\begin{array}{c}
                                                                                 \!\!\!\! N_C\!\!\!\! \\
                                                                                      \!\!\!\! 2\!\!\!\!
                                                                               \end{array}
\right)-1\right)/2+1$, where $N_C$ is the number of classes. Moreover, for a random coding matrix with its entries uniformly chosen, the expected value of any two different class codes is
$\Delta^{Ham(RAND)}(M)$ is $N_L/2$, where $N_L$ is
the code length.
A larger $\Delta^{Ham(RAND)}(M)$ is more likely to rectify the errors committed by individual base classifiers.
Therefore, when $N_L\gg N_C$, a random ECOC is expected to be more robust and rectify more errors than the OVO and OVA approaches ~\cite{Allwein:2001:RMB:944733.944737}. However, the choice of only either binary or ternary codes hinders the construction of longer and more discriminative codes. For example, binary codes can only construct codes of length $N_L \leq 2^{N_C}$. Moreover, they lead to many redundant base learners \cite{4668347}.   In contrast, for $N$-ary ECOC,  the expected value of $\Delta^{Ham(N)}(M)$ is $N_L (1-\frac{1}{N})$ (see Lemma \ref{them:ham} for proof).   $\Delta^{Ham(N)}(M)$ is expected to be larger than $\Delta^{Ham(RAND)}(M)$ when $N\geq 3$.\vspace{-1 mm}
\begin{table}[h]%\vspace{-0.3cm}
\renewcommand*\arraystretch{1.1}
\small
\caption{\label{tab:HamDist} Comparison of  Distance of Different Codes. }
\centering
\begin{tabular}{|c|c|c|c|}
  \hline
Coding  & Generalized & Absolute    \\
Schemes &Hamming Distance& Distance \\\hline
OVA       & $2$ & $4$ \\\hline
OVO   & $\left(\left(\begin{array}{c}
                                                                                \!\!\!\! N_C\!\!\!\! \\
                                                                                 \!\!\!\! 2\!\!\!\!
                                                                               \end{array}
\right)-1\right)/2+1$ & $2N_C -2$ \\\hline
ECOC   & $N_L/2$  & $N_L$\\\hline
$N$-ary ECOC & $N_L (1-1/N)$ & $N_L (N^2-1)/3N$\\\hline
\end{tabular}
\end{table}%\vspace{-1 mm}
\begin{lemma}\label{them:ham}
The expected  Hamming  distance for any two distinct rows in a random $N$-ary coding matrix $M \in \{1,2, \cdots, N\}^{N_C \times N_L}$ is
\begin{equation}
\Delta^{Ham(N)}(M) = N_L (1-\frac{1}{N}).\label{eq:nary_dist_ham}\vspace{-3mm}
\end{equation}
\end{lemma}
\begin{proof}
Given a random matrix $M$ with components chosen uniformly over $\{1,2, \cdots, N\}$, for any distinct pair of entries in  column $s$, i.e., $M(r_i, s)$ and $M(r_j, s)$, the probability of $M(r_i,s) = M(r_j,s)$ is $\frac{1}{N}$.
Then   the probability of $M(r_i,s) \neq M(r_j,s)$ is  $1 - \frac{1}{N}$.

Therefore, according to  Definition \ref{dfn:Ham}, the expected  Hamming  distance for $M$ can be computed as follows,
 \begin{eqnarray}
\Delta^{Ham(N)}(M) &=& N_L \left(1\times(1-\frac{1}{N})+0\times\frac{1}{N}\right)\nonumber\\
 &=& N_L (1-\frac{1}{N}).\nonumber
\end{eqnarray}
\end{proof}%\vspace{-2 mm}
\noindent\textbf{Absolute Distance:}
Different from the Hamming distance, the absolute distance measures the  difference scales.   Thus, for a fair comparison, we assume that  coding values are in the same scale for the absolute distance analysis. The definition of absolute distance is given as follows,
\begin{deftn}[Absolute Distance]\label{dfn:abs} Let $M(r_1,:)$ and $ M(r_2,:)$ denote row $r_1$ and $ r_2$ coding vectors in coding matrix $M$ with length $N_L$, respectively. Then the absolute distance can be expressed as
\begin{eqnarray}
&&\Delta^{abs}(M(r_1,:),M(r_2,:))  = \sum_{s=1}^{N_L} |M(r_1,s)-M(r_2,s)|.\nonumber%\vspace{-3mm}
\end{eqnarray}
\end{deftn}
\vspace{-3 mm}
 For the convenience of analysis, we first give the expected absolute distance for $N$-ary coding matrix in Lemma \ref{them:abs}.
\begin{lemma}\label{them:abs}
The expected absolute distance for any two distinct rows in a random $N$-ary coding matrix $M\in \{1,2, \cdots, N\}^{N_C \times N_L}$ is
\begin{equation}%\vspace{-4mm}
\Delta^{abs(N)}(M) = N_L \frac{(N^2-1)}{3N}.\label{eq:nary_dist}%\vspace{-1mm}
\end{equation}
\end{lemma}

\begin{proof}
Given a random matrix $M$ with components chosen uniformly over $\{1,2, \cdots, N\}$, for any distinct pair of entries in  column $s$, i.e., $M(r_i, s), M(r_j, s)$, we denote the corresponding expected absolute distance as $\Delta^{abs(N)} (M(:,s)) = \E ~{d_{ij}} =  \E~ |M(r_i,s) - M(r_j,s)|$.

It can be calculated by averaging all the possible pairwise distances $d_{ij}$ for $i,j \in \{1,2,\cdots, N\}$.
Since the two numbers $r_i, r_j$ are chosen randomly from $\{1,...,N\}$,
 $\Delta^{N}(M)$ can be expressed as follows:
\begin{eqnarray}
\Delta^{abs(N)}(M(:,s)) &=& \frac{1}{N^2} \sum_{i,j=1}^N d_{ij} \nonumber
\\ &=& \frac{1}{N^2} \sum_{i,j=1}^N |M(r_i,s) - M(r_j,s)|
\end{eqnarray}

\begin{table}[h]
\caption{\label{tab:d_nary} All Possible Choices of $d_{ij}$.}
\normalsize
\centering
\begin{tabular}{|c|c|c|c|c|}
  \hline
  $d_{ij}$ & $r_j = 1$ & $r_j = 2$ & $\cdots$& $r_j = N$ \\\hline
  $r_i = 1$        & 0& 1& $\cdots$& N-1\\
  $r_i = 2$         & 1& 0& $\cdots$& $\vdots$\\
  $\vdots$ & $\vdots$& $\vdots$& 0& 1\\
  $r_i = N$        & N-1& $\cdots$& 1& 0\\
  \hline
\end{tabular}
\end{table}
First, we define the  sequence $a_n$ as follows:
\begin{eqnarray}
a_n =  (1+2+\cdots+n)  = \frac{n(n+1)}{2}.
\end{eqnarray}
Table \ref{tab:d_nary} gives  all the possible choices of $d_{ij}$. Thus the calculation of $\Delta^{N}(M)$ is equal to taking the average of all the entries in  Table \ref{tab:d_nary}, which can be expressed as follows:
\begin{eqnarray}
\Delta^{abs(N)}(M(:,s))\!\!\!\!\!\!&=&\!\!\!\!\!\!\frac{2}{N^2} (a_1 + a_2 + \cdots + a_{N-1}) \label{eq:nary_dist_sym}\\
\!\!\!&=&\!\!\!\!\!\!\frac{1}{N^2} (1\times 2+ 2\times 3 +\cdots+ (N-1)N) \nonumber\\
\!\!\!&=&\!\!\!\!\!\!\frac{1}{N^2} \sum_{n=1}^N (n^2-n)\nonumber\\
\!\!\!&=&\!\!\!\!\!\!\frac{1}{N^2} \left(\sum_{n=1}^N n^2 - \sum_{n=1}^N n\right)\nonumber\\
\!\!\!&=&\!\!\!\!\!\!\frac{1}{N^2} \left(\frac{N(N+1)(2N+1)}{6} - \frac{N(N+1)}{2}\right)\nonumber\\
\!\!\!&=&\!\!\!\!\!\!\frac{N^2-1}{3N},\nonumber
\end{eqnarray}
where (\ref{eq:nary_dist_sym}) comes from the symmetry of $d_{ij}$.
Then
\begin{eqnarray}
\Delta^{abs(N)}(M) =  \sum_{s=1}^{N_L} \Delta^{abs(N)}(M(:,s)) = N_L \frac{(N^2-1)}{3N}.\nonumber\end{eqnarray}
\end{proof}
For the OVA coding scheme,  every two rows have exactly two
entries with opposite signs, the minimum absolute distance
$\Delta_{min}^{abs(OVA)}(M)=4$; while for the OVO coding scheme,  every two
rows have exactly one entry with opposite signs and only $2N_C-4$
entries with a difference of exactly one,
$\Delta_{min}^{abs(OVO)}(M)=2N_C -2$. For  binary random codes, the
expected absolute distance between any two different rows is $\Delta^{abs(RAND)}(M)=N_L$. Thus, when $N$ is large, $\Delta^{abs(N)}(M)$  is much
larger than $\Delta^{abs(RAND)}(M)$, and $N$-ary coding is expected to be
better.

The Hamming and absolute distance comparisons for different codes are summarized in the Table \ref{tab:HamDist}. We can see that $N$-ary coding scheme has  an advantage in creating more discriminative codes with larger distances for different classes in both two distance measures. This advantage is very important to analyze the generalization error analysis of $N$-ary ECOC.

\section{Complexity Comparison }\label{sec:complexity}
As discussed in Section \ref{sec:necoc},  $N$-ary codes have a better error correcting capability than the traditional random codes
when $N$ is larger than 3.   However, one notes that the base classifier of each column is no longer  solving a binary problem. Instead, we randomly break a large multi-class problem into a number of smaller multi-class subproblems. These subproblems are more complicated than binary problems and they incur additional computational cost. Hence, there is a trade-off between the error correcting capability and computational cost.

If the complexity of the algorithm employed to learn the  small-size multi-class base problem is $\mathcal{O}(g(N,N_{tr},d))$  with $N$ classes, $N_{tr}$ training examples, $d$ predictive features and $g(N,N_{tr},d)$ is the complexity function w.r.t $N$, $N_{tr}$, $d$,  then the computational complexity of $N$-ary codes is $\mathcal{O}(N_L g(N,N_{tr},D))$  for codes of length $N_L$.

Taking SVM as the base learner for example, one can learn each binary classification task created by binary ECOC codes with training complexity of  $\mathcal{O}(N_{tr}^3)$ for traditional SVM solvers that build on the quadratic programming (QP) problems. However, a major stumbling block for these traditional methods is in scaling up these QP¡¯s to large data sets, such
as those commonly encountered in data mining applications. Thus, some state-of-the-art SVM implementations, e.g., LIBSVM \cite{Chang:2011:LLS:1961189.1961199}, Core Vector Machines \cite{DBLP:journals/jmlr/TsangKC05},  have been proposed to reduce training time complexity from $\mathcal{O}(N_{tr}^3)$ to $\mathcal{O}(N_{tr}^2)$ and $\mathcal{O}(N_{tr})$, respectively.  Nevertheless, how to efficiently train SVM is not the focus of our paper. For the convenience of complexity analysis, we use the time complexity of the traditional SVM solvers as the complexity of the base learners. Then, the complexity of  binary ECOC codes is $\mathcal{O}(N_L N_{tr}^3)$. Different from ECOC in the ensemble manner, one can directly address the  multi-class problem in one single optimization process, e.g., multi-class SVM \cite{Crammer:2002:AIM:944790.944813}. This kind of model combines multiple binary-class optimization problems into one single objective function and simultaneously achieves the classification of multiple
classes. In this way, the correlations across multiple binary classification tasks are captured in the learning model. The resulting QP optimization requires a complexity of $\mathcal{O}((N_C N_{tr})^3)$.   However, it causes  high computational complexity for a relatively large number of classes. In contrast, $N$-ary codes are in the complexity of $\mathcal{O}( N_L (NN_{tr})^3)$, where $N < N_C$.  In this case, it achieves  better trade-off between the error correcting capability and computational cost, especially for large class size $N_C$.

We summarize the time complexity of different codes in Table \ref{tab:Compelxity}. In Section \ref{sec:acc:vs:N}, our empirical studies  indicate that $N$ does not need to be too large to achieve optimal classification performance.

%For decision tree algorithms like  CART \cite{cart84-2} and C4.5 \cite{Quinlan:1993:CPM:152181},
%They can be easily generalized to handle these multi-class learning
%tasks. Each leaf of the decision tree can be labeled with one of the $C$ classes, and internal
%nodes can be selected to discriminate among these classes. Thus, the decision tree algorithms usually require time complexity of $N_{tr} d^2$.  The complexities of both $N$-ary codes and binary codes are $\mathcal{O}(N_L N_{tr}d^2)$ \cite{Su:2006:FDT:1597538.1597619}. It is noted that for base learners like decision tree, the complexity is not reduced
%

%\begin{table}[h]
%\caption{\label{tab:Compelxity} Complexity Analysis of $N$-ary ECOC}
%\small
%\centering
%\begin{tabular}{|c|c|c|c|}
%  \hline
%Classifier & Binary ECOC & Direct Multi-Class  & $N$-ary ECOC\\ \hline
%SVM        & $\mathcal{O}(2 N_L N_{tr}d)$ &$\mathcal{O}(C N_{tr}d)$ &$\mathcal{O}(N_L N N_{tr}d)$ \\\hline
%CART       & $\mathcal{O}(N_L N N_{tr}d^2)$ & $\mathcal{O}(N_L N N_{tr}d^2)$ & $\mathcal{O}(N_L N N_{tr}d^2)$\\
%  \hline
%\end{tabular}
%\end{table}

\begin{table}[h]
\renewcommand*\arraystretch{1}
\normalsize
\caption{\label{tab:Compelxity} Complexity Analysis}
\centering
\begin{tabular}{|c|c|c|}
  \hline
Classifier & SVM  \\ \hline
Binary ECOC       & $\mathcal{O}(N_L N_{tr}^3)$  \\\hline
Direct Multi-Class   & $\mathcal{O}((N_C N_{tr})^3)$  \\\hline
$N$-ary ECOC &$\mathcal{O}( N_L (NN_{tr})^3)$ \\\hline
\end{tabular}
\end{table}

\section{Generalization Analysis of $N$-ary ECOC.}\label{sec:Bound}
%This section, we are going to give theoretical insights of $N$-ary ECOC.
%Although a comprehensive analysis  \cite{Allwein:2001:RMB:944733.944737} of the %generalization error of binary and ternary code-based ECOC for multi-class problems has been %conducted, the analysis cannot be directly applied to $N$-ary codes.
In Section~\ref{sec:correct}, we study the error correcting ability of an $N$-ary code. %Then we will investigate the generalization error bound for the proposed OR-ECOC framework .
In Section~\ref{sec:geecoc}, we derive the generalization error bound for  $N$-ary ECOC independent of  the base classifier.

\subsection{Analysis of Error Correcting on $N$-ary Codes}~\label{sec:correct}
To study the error correcting ability of $N$-ary codes, we first define the  distance between the codes in any distinct pair of rows, $M(r_i)$ and $M(r_j)$, in an $N$-ary coding matrix $M$ as $\Delta^N(M(r_i),M(r_j))$. It is the sum of the $N_L$ distances between two entries, $M(r_i, s)$ and $M(r_j, s)$ in the same column $s$ at two different rows, $r_i$ and $r_j$, i.e.,
$
\Delta^N(M(r_i),M(r_j)) = \sum_{s=1}^{N_L} \Delta^N(M(r_i,s),M(r_j,s)).
$
%where $d(M(r_1,s),M(r_2,s)) = |M(r_1,s)-M(r_2,s)|$. % is the absolute distance between  $M(r_1,s)$ and $M(r_2,s)$.
We further define $\rho = \min_{r_i\neq r_j}{\Delta^N(M(r_i),M(r_j))}$ as the minimum distance between any two rows in $M$.

\begin{proposition}\label{thm:min_margin}
Given an $N$-ary coding matrix $M$ and a vector of predicted labels $f(\bx)=[f_1(\bx)),\cdots,f_{N_L}(\bx))]$ by $N_L$ base classifiers for a test instance $\bx$. If $\bx$
is misclassified by the $N$-ary ECOC decoding, then the distance between the correct label in $M(y)$ and  $f(\bx)$ is greater than one half of $\rho$, i.e.,
\begin{equation}\label{eq:margin}
  \Delta^N(M(y),f(\bx))
\geq \frac{1}{2} \rho.
\end{equation}
\end{proposition}
\begin{proof}
Suppose that the distance-based decoding incorrectly classifies a test instance $\bx$ with known label $y$. In other words, there exists a label $r\neq y$
for which
\begin{equation*}
 \Delta^N(M(y),f(\bx))\geq \Delta^N(M(r),f(\bx)).
\end{equation*}

Here, $\Delta^N(M(y),f(\bx))$ and $\Delta^N(M(r),f(\bx))$ can be expanded as the element-wise summation. Then, we have
\begin{eqnarray}
~~\sum_{s=1}^{N_L}
\Delta^N(M(y,s),f_s(\bx)) \geq \sum_{s=1}^{N_L} \Delta^N(M(r,s),f_s(\bx)). \label{eq:sum}
\end{eqnarray}
%where $f_s(\bx)$ is the predicted label at column $s$.
Based on the above inequality, we obtain:
%\begin{small}
%\vspace{-1mm}
\begin{eqnarray}
 && \Delta^N(M(y),f(\bx)) \nonumber \\ \!\!\!\!\!
=&& \!\!\!\!\! \frac{1}{2}
\sum_{s=1}^{N_L} \left\{\Delta^N(M(y,s),f_s(\bx))\!+ \!\Delta^N(M(y,s),f_s(\bx))\right\}\nonumber\\
\!\!\!\!\geq&& \!\!\!\!\!\!\frac{1}{2}
\sum_{s=1}^{N_L} \{\Delta^N(M(y,s),f_s(\bx))\!+ \! \Delta^N(M(r,s),f_s(\bx))\}\label{eq:relax_1}\\
\!\!\!\!\geq&& \!\!\!\!\!\!\frac{1}{2}
\sum_{s=1}^{N_L} \{\Delta^N(M(y,s),M(r,s))\} \!\!\label{eq:relax_2}\\
 \!\!\!\!=&&\!\!\!\!\!\!\frac{1}{2}{\Delta^N(M(r),M(y))}\nonumber\\
 \!\!\!\!\geq &&\!\!\!\!\!\!\frac{1}{2}\rho\nonumber,
\end{eqnarray}
%\vspace{-0mm}
%\end{small}
where Inequality (\ref{eq:relax_1}) uses  Inequality (\ref{eq:sum}) and Inequality (\ref{eq:relax_2}) follows from the
triangle inequality.\end{proof}%See supplementary materials.
\vspace{-0mm}

\begin{remark}
From Proposition~\ref{thm:min_margin}, one notes that
a mistake on a test instance $(\bx,y)$
implies that  $\Delta^N(M(y),f(\bx)) \geq \frac{1}{2}\rho$. In other words, the prediction codes are not required to be exactly the same as the ground-truth codes for all the base classifications. As long as the distance is no larger than $\frac{1}{2}\rho$, $N$-ary ECOC can rectify the error committed by some base classifiers, and is still able to make an accurate prediction. This error correcting ability is very important especially when the labeled data is insufficient. Moreover, a larger minimum distance, i.e., $\rho$, leads to a stronger capability of error correcting. Note that this proposition holds for all the distance measures and traditional ECOC schemes due to the fact that only the triangle inequality is required in the proof.
\end{remark}

\subsection{Generalization Error of $N$-ary ECOC.} \label{sec:geecoc}
The next result provides a generalization error bound for \emph{any} type of base classifier, such as the SVM classifier and decision tree, used in the $N$-nary ECOC classification.

\begin{theorem}[$N$-ary ECOC Error Bound] \label{thm:N_ary_error}
Given $N_L$ base classifiers, $f_1,
\cdots, f_{N_L}$, trained on $N_L$ subsets $\{(\bx_i,
M(y_i,s))_{i=1,\cdots,N_{tr}}\}_{s=1,\cdots,N_L}$ of the dataset with $N_{tr}$ instances for coding matrix
$M\in\{1,2,\cdots,N\}^{N_C\times N_L}$. The generalized error rate for the $N$-ary ECOC approach using distance-based
decoding is upper bounded by
\begin{equation}
 \frac{2 N_L \bar{B}}{\rho },
 \end{equation}
where $\bar{B} = \frac{1}{N_L}\sum_{s=1}^{N_L} B_s$ and $B_s$ is the upper bound of the distance-based loss for the $s^{th}$ base classifier.
\end{theorem}
\begin{proof}
According to Proposition~\ref{thm:min_margin}, for any
misclassified data instance, the distance between its
incorrect label vector $f(\bx)$ and the true label vector $M(y)$ should
satisfy the minimal distance $\frac{\rho}{2}$, i.e.,
$
\Delta^N(M(y),f(\bx)) = \sum_{s=1}^{N_L}
 \Delta^N(M(y,s),f_s(\bx)) \geq \frac{\rho}{2}.
$

Let $a$ be the number of incorrect label predictions for a set of test instances of size $N_{te}$.
One obtains
\begin{eqnarray}
a\frac{\rho}{2} \leq \sum^{N_{te}}_{i=1}\sum^{N_L}_{s=1} \Delta^N(M(y_i,s),f(\bx_i)).
\end{eqnarray}
 Then,
  \begin{eqnarray}
 a \leq \frac{2N_{te}\sum^{N_L}_{s=1} B_s} {\rho} =  \frac{2 N_{te} N_L \bar{B}}{\rho},
\end{eqnarray}
where $\bar{B} = \frac{1}{N_L}\sum_{s=1}^{N_L} B_s$.

Hence, the testing error rate is bounded by $\frac{2 {N_L} \bar{B}}{\rho}$.
\end{proof}
\begin{remark} From Theorem~\ref{thm:N_ary_error}, one notes that for the fixed $N_L$, the generalization
error bound of the $N$-ary ECOC  depends on the two following factors:
 \begin{enumerate}
  \item The averaged  loss $\bar{B}$ for all the base classifiers.  In practice, some base tasks may be badly designed due to the randomness. As long as the averaged loss $\bar{B}$ over all the tasks is small, the resulting ensemble classifier is still able to make a precise prediction.

  \item The minimum distance $\rho$ for coding matrix $M$. As we discussed in Proposition~\ref{thm:min_margin}, the larger $\rho$, the stronger capability of error correcting $N$-ary code has.
  \end{enumerate}
  Both two factors are affected by the choice of $N$. In particular,  $\bar{B}$  increases as $N$ increases since the base classification tasks become more difficult.
%Moreover, as discussed in  Section \ref{sec:necoc}, the complexity increases with $N$. On the %other hand,
On the other hand, from experimental results in Figure \ref{fig:rho_ham}, it is observed that $\rho$ becomes larger when $N$ increases. Therefore, there is a tradeoff between these two factors.

\end{remark}

\section{Experimental Results} \label{sec:experiments}
We present experimental results on 11 well-known UCI multi-class datasets from a wide range of application domains. The statistics of these datasets are summarized in Table
\ref{tab:summary_datasets}. %In particular, the Auslan and Sector datasets have around 100 classes, which are more challenging than the remaining datasets.
 The parameter $N$ is chosen by cross-validation procedure. With the tuned parameters, all methods are run ten realizations. Each  has different random splittings with fixed  training and testing size  as given in Table  \ref{tab:summary_datasets}.
 Our experimental results focus on the comparison of different encoding schemes rather than decoding schemes. Therefore, we fix generalized hamming distance  as the decoding strategy for all the coding designs for a fair comparison.

\begin{table}[h]
\renewcommand*\arraystretch{1}
\centering \caption{\label{tab:summary_datasets} Summary of the datasets used in the experiments.}
{\small
\begin{tabular}{c| c c c c }
  \hline
    Dataset    &\#Train  &\#Test  & \#Features & \#Classes\\\hline
    Pendigits  & 3498    &7494    & 16         & 10\\
    Vowel      & 462     &528      & 10         & 10\\
    News20     & 3993    &15935   & 62061      & 20\\
    Letters    & 5000    &15000   & 16         & 26\\
    Auslan     & 1000    &1565        & 128        & 95 \\
    Sector     & 3207    &6412        & 55197      & 105\\
    Aloi       & 50000   &58000     & 128          &1000\\
    Glass      & 100     &114         & 9          & 10\\
    Satimage   & 3435    &3000    & 36         & 7 \\
    Usps       & 4298    &5000    & 256        & 10\\
    Segment    & 1310    &1000    & 19         & 7
\end{tabular}
%\vspace{-4mm}
}
\end{table}

To investigate the effectiveness of
the proposed $N$-ary coding scheme, we compare it  with
problem-independent coding schemes including OVO, OVA, and random
ECOC as well as the state-of-art problem-dependent methods such as
ECOCONE and DECOC.   For the  random ECOC encoding scheme, or ECOC in short, and the $N$-ary ECOC  strategy,  we select the matrix with the largest minimum absolute distance
from 1000 randomly generated matrices.

For the problem-dependent approach DECOC, the length of the ECOC codes
is exactly $N_C-1$ \cite{Pujol06discriminantecoc:}. For the
ECOCONE, we initialize the ECOC matrix with OVA matrix
\cite{conf/icpr/EscaleraP06}. The length of the ECOC code is also
learned during the training step. We use the ECOC library
\cite{Escalera:2010:EOC:1756006.1756026} for the implementation of
all these baseline methods. To ensure a fair
comparison and easy replication of results, the base learners decision tree CART \cite{cart84-2} and linear SVM are implemented with the CART decision tree MATLAB toolbox and the LIBSVM \cite{Chang:2011:LLS:1961189.1961199} with the linear kernel in default settings, respectively.

\subsection{Error Bound Analysis on $N$-ary ECOC.}\label{sec:effect_code}
In the bound analysis, we choose  hamming distance \ref{dfn:Ham} to measure the row separation as a showcase.  % due to absolute distance will introduce the an ordering among the classes.
According to  Theorem \ref{thm:N_ary_error},  the generalization error bound depends on the minimum distance $\rho$ between any two distinct rows in the $N$-ary coding matrix $M$ as well as the average loss of base classifiers $\bar{B}$. In particular, the expected value of $\Delta^N(M)$ scales with $O(N)$.

In this subsection, we investigate the effect of the number
of  classes $N$ using the Pendigits dataset with CART  as the base classifier to illustrate the following aspects: (i)  $\Delta^{N}(M)$ between any two
distinct rows of codes (see Figure~\ref{fig:aver_distance_ham} ), (ii) $\rho$ (see Figure~\ref{fig:rho_ham}), (iii) $\frac{\bar{B}}{\rho}$ (see Figure ~\ref{fig:ratio_ham}), and (iv) the classification performance (see Figure~\ref{fig:acc_Impact_N}).   The empirical results corroborate the proposed error bounds in Theorem \ref{thm:N_ary_error}.

\begin{figure*}[htp!]
%\vspace{-3.5 cm}
\centering
\subfigure[  Average $\Delta^N(M)$ v.s. $N$.] {\label{fig:aver_distance_ham}\includegraphics[trim = 3.5cm 9.5cm
4.1cm 5.5cm, clip, width=0.25\textwidth]{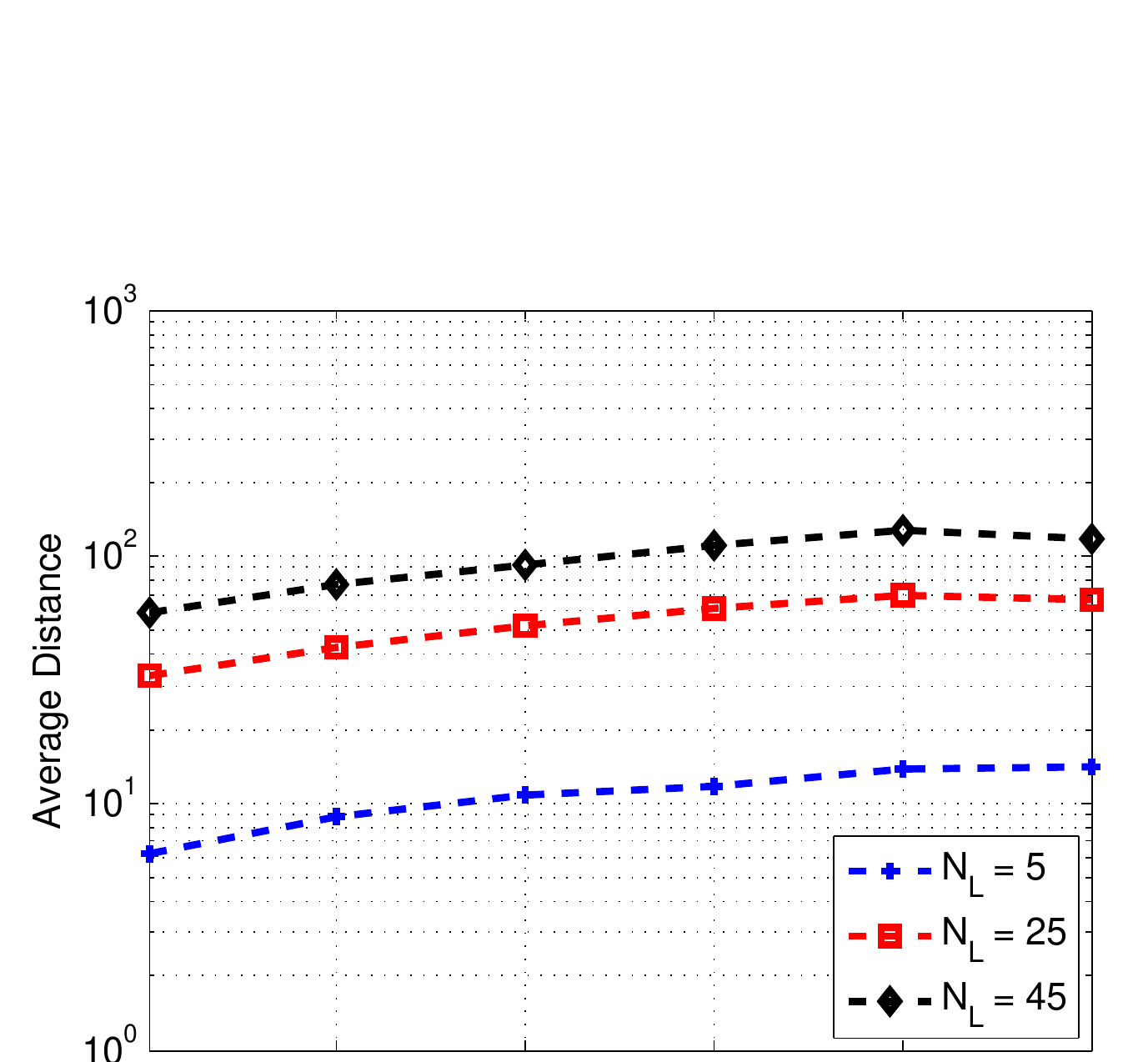}}%\vspace{-1.5cm}
\subfigure[ $\rho$ v.s. $N$.] {\label{fig:rho_ham}\includegraphics[trim = 3.5cm 9.5cm
4.1cm 5.5cm,clip, width=0.25\textwidth]{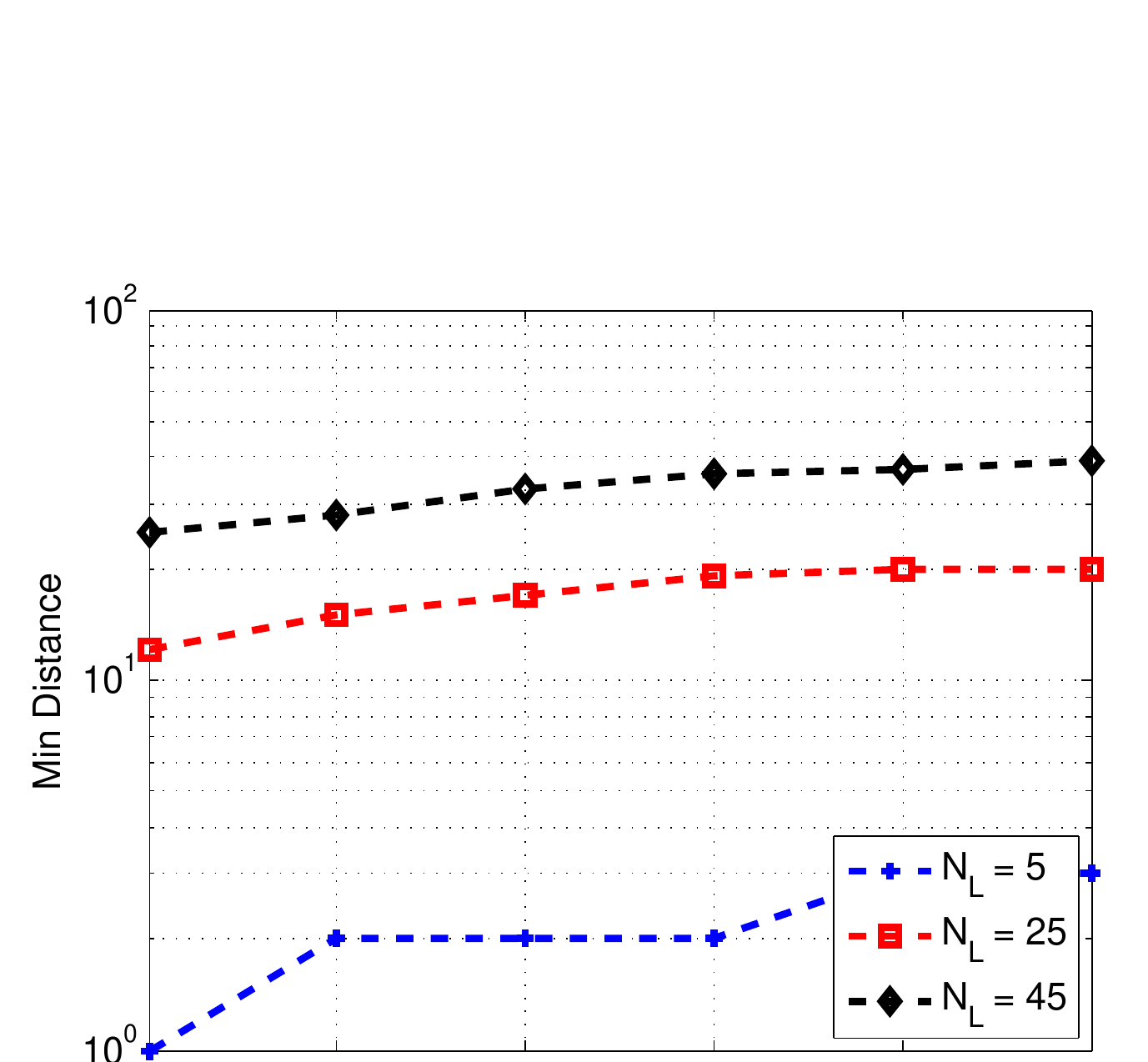}}%\vspace{-0.1 cm}
 \subfigure[ $\frac{\bar{B}}{\rho}$
v.s. $N$.]{\label{fig:ratio_ham}\includegraphics[trim = 3.5cm 9.5cm
4.1cm 5.5cm, clip, width=0.25\textwidth]{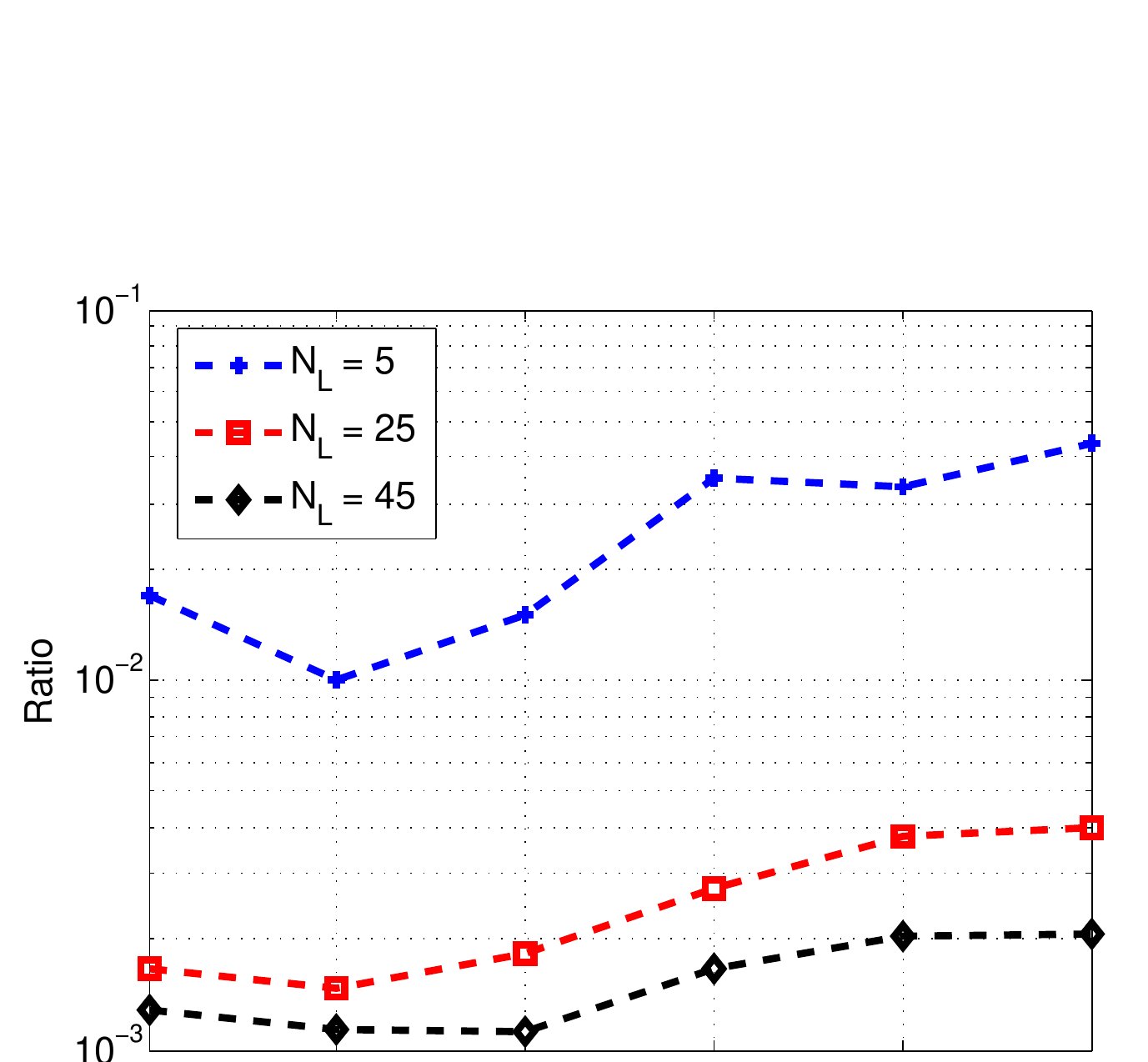}}%\vspace{-0.4 cm}
\caption{Experimental results  to study error bound (Theorem \ref{thm:N_ary_error}) w.r.t.  $N$.} \label{fig:bound_Impact_N_ham}
\end{figure*}

\subsubsection{Average distance
$\Delta^N(M)$ v.s. $N$.}\label{sec:Aver_dist}
Recall that the hamming distance for different coding matrices discussed in Section \ref{sec:necoc} are:  $\Delta^{N}(M)=N_L (1-\frac{1}{N})$,
$\Delta^{rand}(M)=N_L/2$,
 $\Delta_{\min}^{ova}(M)=2$ and $\Delta_{\min}^{ovo}(M)=\left(\left(\begin{array}{c}
                                                                                \!\!\!\! N_C\!\!\!\! \\
                                                                                 \!\!\!\! 2\!\!\!\!
                                                                               \end{array}
\right)-1\right)/2+1$.

From Figure \ref{fig:aver_distance_ham}, we observe that  the
empirical average hamming distances of the constructed $N$-ary
coding matrices for random $N$-ary schemes % described in Section~\ref{sec:encoding_matrix}
are close to $N_L (1-\frac{1}{N})$.
Furthermore, when there are 45 base classifiers, the average
distance for $N$-ary ECOC is larger than 30, which is
 larger than that of the binary random codes with an average absolute distance of 22.5. Moreover, a higher
$N$ leads to a larger average distance.
 Comparing  Figure \ref{fig:aver_distance_ham} and Figure \ref{fig:rho_ham} , the large average
distance $\Delta^N(M)$ also correlates with the large minimum
distance $\rho$.

\subsubsection{Minimum distance $\rho$ v.s. $N$.}\label{sec:Min_dist}
For the Pendigits dataset with 10 classes, $\rho$ for OVA and OVO are 4 and 18, respectively.
From Figure \ref{fig:rho_ham},  we
observe that with a fixed number of base classifiers, $\rho$
 increases with the number of multi-class subproblems  of class-size $N$, meanwhile $\rho$ also increases with respect to
the code length $N_L$.
Furthermore, in comparison to the other coding schemes, our
proposed method usually creates a coding matrix with a large $\rho$.
 For example, in Figure \ref{fig:rho_ham}, one observes that when there are 25
and 45 base classifiers, the corresponding $\rho$ for binary random
codes are 0. On the other hand, $N$-ary ECOC, given a
sufficiently large $N$, creates an $N$-ary coding  matrix with
$\rho$ to be larger than 10 and 20, respectively.
Although $N$-ary ECOC creates
an $N$-ary coding matrix with a large $\rho$ when  $N$ is larger, in real-world applications,
 it is preferred that $N$ is not too large to ensure reasonable computational cost and difficulty of  subproblems.
In short, $N$-ary ECOC provides a better alternative to creating a
coding matrix with a large class separation compared to traditional coding
schemes.

\subsubsection{Ratio $\bar{B}/\rho$ v.s. $N$.}
Both $\bar{B}$ and $\rho$ are dependent on $N$. Moreover, from the generalization error bound, we observe that $\bar{B}/\rho$ directly affects  classification performance.
%The previous experimental results  show that with increasing $N$, the distance (almost) monotonically increase.
Hence, this ratio, which bounds the classification error, requires further investigation.
Figure \ref{fig:ratio_ham} shows that when $N = 4$, the ratio $\bar{B}/\rho$ is lowest. This        observation suggests that the more the row and column separation of the coding matrix, the stronger the capability  of
error correction  \cite{Dietterich95solvingmulticlass}.
Therefore, $N$-ary ECOC is a better way to creating the coding matrix with
large separation among the classes as well as more diversity, compared to the binary
and ternary coding schemes. One notes that $\bar{B}/\rho$ starts to increase when $N \geq 5$. This means that the increase of  the average base classifier loss $\bar{B}$ overwhelms the increase in $\rho$. The reason for this phenomena is the increase in difficulty of the subproblem classification with more classes.

\subsubsection{Classification Accuracy v.s. $N$.}\label{sec:acc:vs:N}
Next, we study the impact of $N$ on the multi-class classification accuracy. We use datasets Pendigits, Letters, Sectors, Aloi with 10 classes, 26 classes, 105 classes, 1000 classes respectively  as showcase. %\footnote{We only report the results on three datasets. However, we have the similar observations on the other datasets.}
  In order to a obtain meaningful analysis, we choose a suitable classifier for different datasets. In particular, we apply the CART to datasets Pendigits, Letters and Aloi and linear SVM to Sectors. One observes from Figure
\ref{fig:acc_Impact_N} that the
$N$-ary ECOC achieves competitive  prediction performance when $3\leq  N \leq 10$. However, given  sufficient base learners, the classification error starts increasing  when $N$ is large (e.g. $N > 4$ for Pendigits, $N > 5$ for Letters and $N > 8$ for Sector). This is because the base tasks are more challenging to solve when $N$ is large and it indicates the influence of $\bar{B}$ outweighs that of $\rho$.
 Furthermore, one observes that the performance curves in Figure \ref{fig:ratio_ham} and \ref{fig:acc_Pen_N} roughly correlate to each other. Hence, one can estimate the trend in the empirical error using the ratio $\bar{B}/\rho$. This verifies the validity of the generalized error bound in Theorem \ref{thm:N_ary_error}.  To investigate the choice of $N$ on  multi-class classification more comprehensively, we further conduct  experiments on the other datasets.  The results of datasets Pendigits, Letters,  Sectors and Aloi  are summarized in Figure \ref{fig:acc_Pen_N}, \ref{fig:acc_Letter_N}, Figure \ref{fig:acc_Sector_N} and Figure \ref{fig:acc_Aloi_N}, respectively. For the rest of the datasets, we have the similar observations. In general, smaller values of $N$ ($N \in [3,10]$) usually lead to reasonably competitive performance.   In other words, the complexity of base learners for $N$-ary codes does not need to  significantly increase above 3  for the performance to be better than existing binary or ternary coding approaches.
 %This advantage  makes the $N$-ary ECOC very practical in real-world applications.
%In conclusion, in order to achieve a better tradeoff between the discriminability and complexity, %we suggest  choosing $N$ between 3 and 10.

%\vspace{0.3 cm}

\subsubsection{Classification Accuracy v.s. $N_L$.}% \vspace{-1mm}
%In this section, we analyze the performance on varying $N_L$.
From Figure \ref{fig:acc_Impact_L},  we observe that high accuracy  can be achieved with a small number of base learners. Another important observation is that given fewer base learners, it is better to choose a large value of $N$ rather than a small $N$. This may be due to the fact that a larger $N$ leads to stronger discrimination among codes as well as base learners. However, neither a large nor small $N$ can reach optimal results given a sufficiently large $N_L$.

%\begin{figure*}[htp!]
%\vspace{-2.5 cm}
%\centering
%\begin{varwidth}{0.45\linewidth}
%\subfigure[  Pendigits.] {\label{fig:acc_Pen_N}\includegraphics[trim = 1.5cm 6.5cm
%2.1cm 3cm, clip, width=1\columnwidth]{Pendigits_N.pdf}}\vspace{-1.5 cm}
%\subfigure[ Letters.] {\label{fig:acc_Letter_N}\includegraphics[trim = 1.5cm 6.5cm
%2.1cm 3cm, clip, width=1\columnwidth]{Letters_N.pdf}}\vspace{-1.5 cm}
% \subfigure[ Sector.]{\label{fig:acc_Sector_N}\includegraphics[trim = 1.5cm 6.5cm
%2.1cm 3cm, clip, width=1\columnwidth]{Sector_N.pdf}}\vspace{-1.5 cm}
%\subfigure[ Aloi.] {\label{fig:acc_Aloi_N}\includegraphics[trim = 1.5cm 6.5cm
%2.1cm 3cm, clip, width=1\columnwidth]{Aloi_N.pdf}}\vspace{-0.4 cm}
%\caption{ Accuracy vs. $N$.} \label{fig:acc_Impact_N}
%\end{varwidth}
%\begin{varwidth}{0.45\linewidth}
%\subfigure[Pendigits.] {\label{fig:acc_Pen_L}\includegraphics[trim = 1.5cm 7.5cm
%2.1cm 3cm, clip, width=1\columnwidth]{Pendigits_L.pdf}}\vspace{-1.5 cm}
%\subfigure[Letters.] {\label{fig:acc_Letter_L}\includegraphics[trim = 1.5cm 7.5cm
%2.1cm 3cm, clip, width=1\columnwidth]{Letters_L.pdf}}\vspace{-1.5 cm}
% \subfigure[ Sector.]{\label{fig:acc_Sector_L}\includegraphics[trim = 1.5cm 7.5cm
%2.1cm 3cm, clip, width=1\columnwidth]{Sector_L.pdf}}\vspace{-1.5 cm}
% \subfigure[ Aloi.]{\label{fig:acc_Aloi_L}\includegraphics[trim = 1.5cm 7.5cm
%2.1cm 3cm, clip, width=1\columnwidth]{Aloi_L.pdf}}\vspace{-0.4 cm}
%\caption{ Accuracy vs. $N_L$.} \label{fig:acc_Impact_L}
%\end{varwidth}
%\end{figure*}

\begin{figure*}[htp!]
%\vspace{-2.5 cm}
\centering
%\begin{varwidth}{0.45\linewidth}
\subfigure[  Pendigits.] {\label{fig:acc_Pen_N}\includegraphics[trim = 1.5cm 6.5cm
2.1cm 3cm, clip, width=0.25\textwidth]{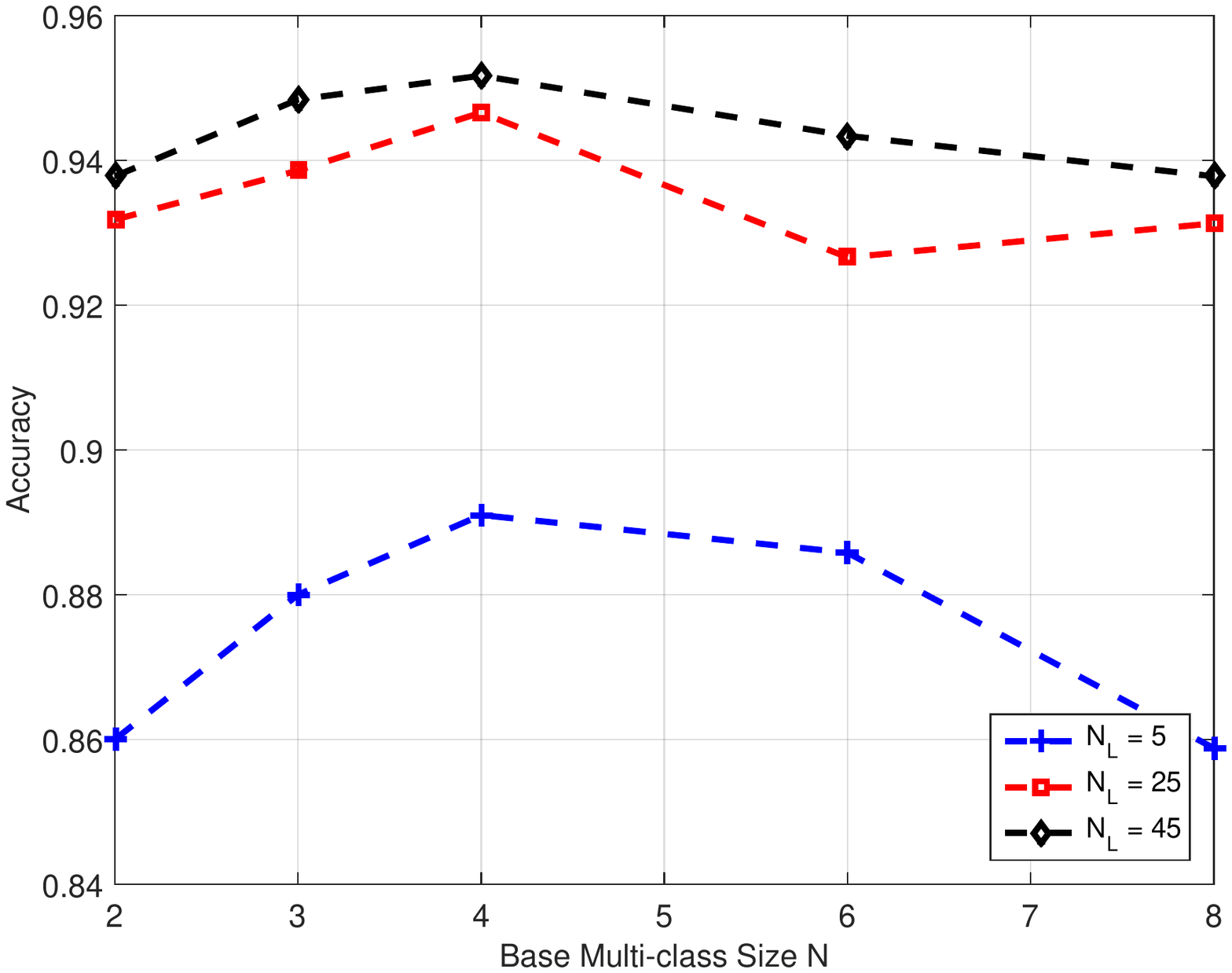}}%\vspace{-0.5 cm}
\subfigure[ Letters.] {\label{fig:acc_Letter_N}\includegraphics[trim = 1.5cm 6.5cm
2.1cm 3cm, clip, width=0.25\textwidth]{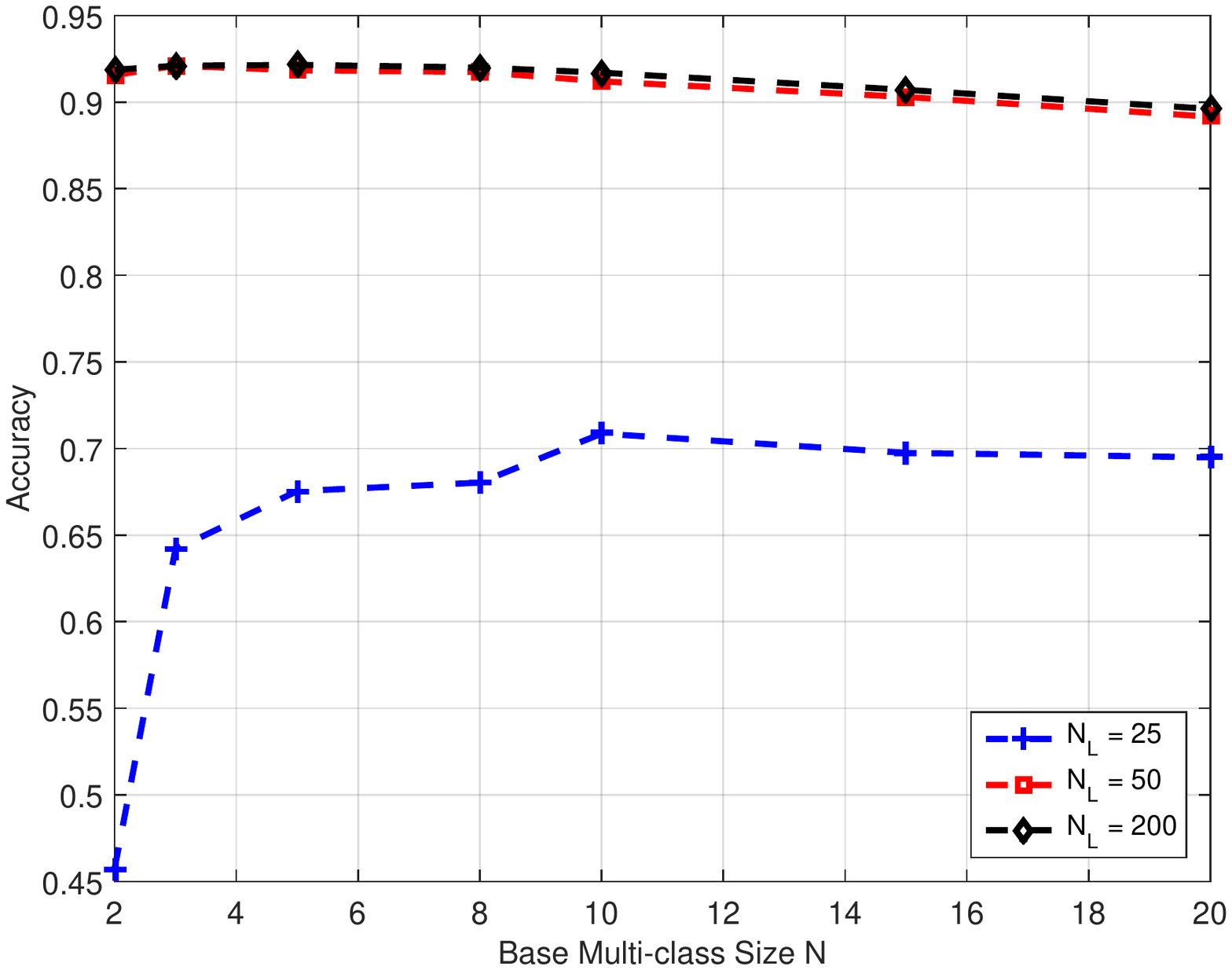}}%\vspace{-0.5 cm}
 \subfigure[ Sector.]{\label{fig:acc_Sector_N}\includegraphics[trim = 1.5cm 6.5cm
2.1cm 3cm, clip, width=0.25\textwidth]{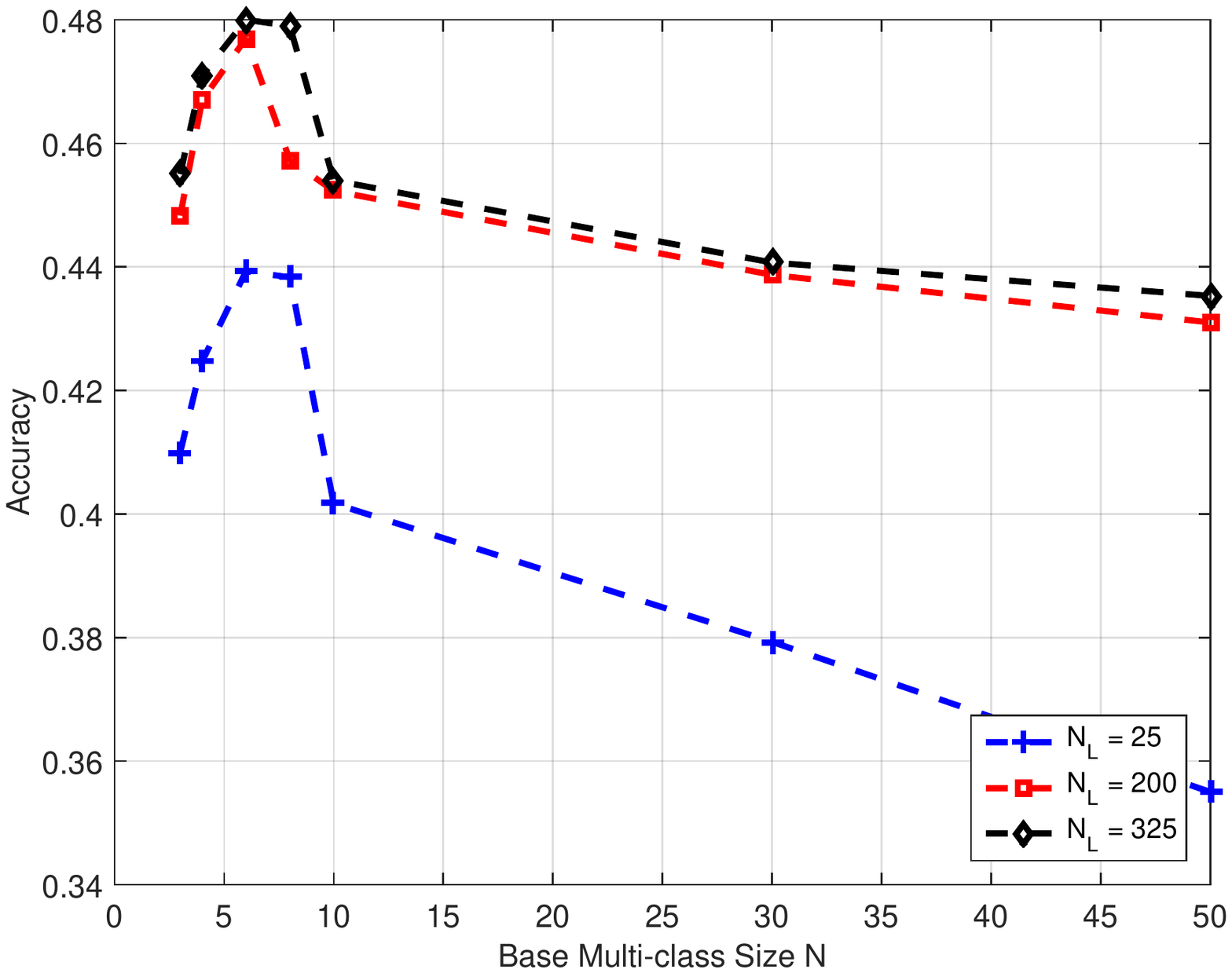}}%\vspace{-0.5 cm}
\subfigure[ Aloi.] {\label{fig:acc_Aloi_N}\includegraphics[trim = 1.5cm 6.5cm
2.1cm 3cm, clip, width=0.25\textwidth]{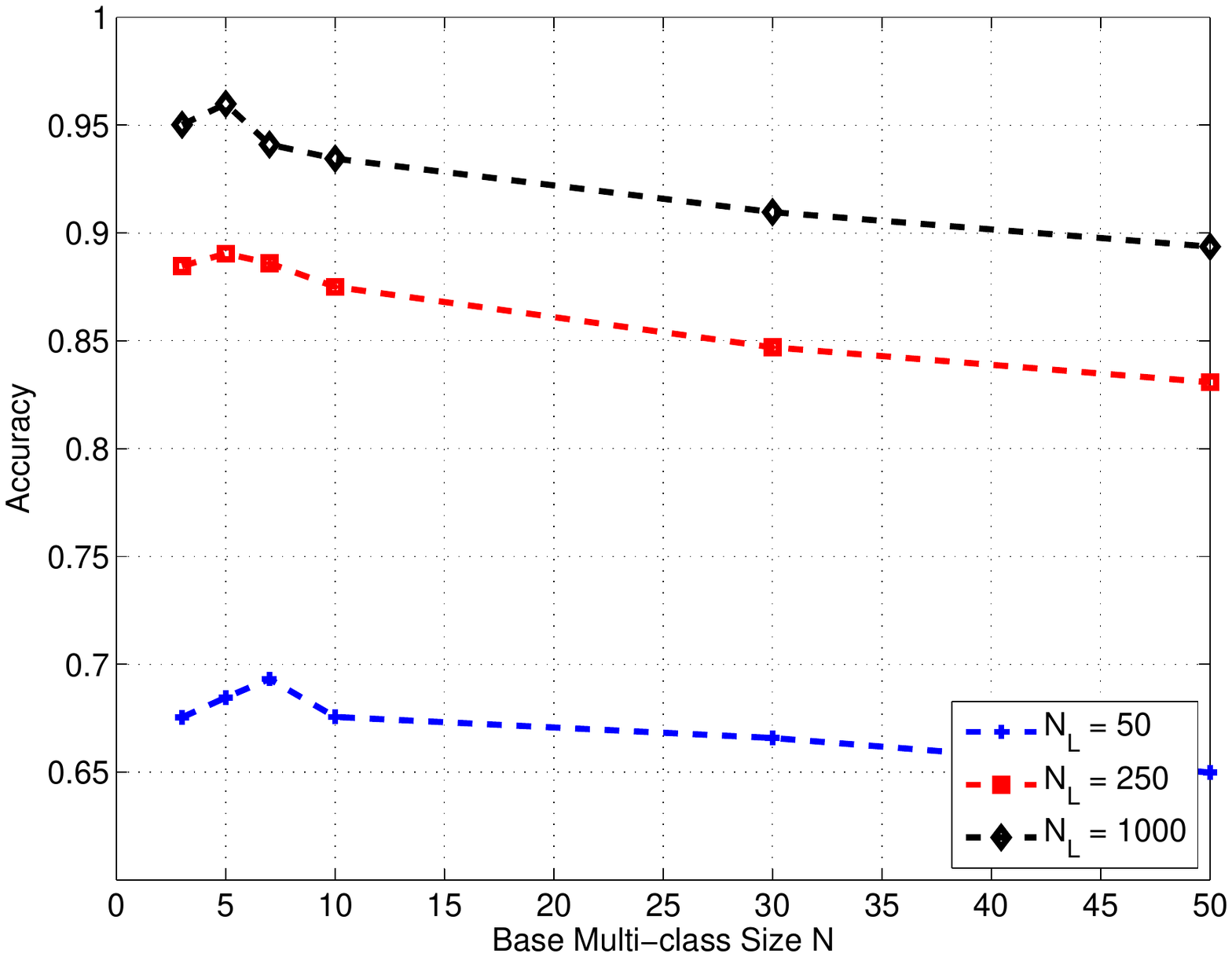}}%\vspace{-0.05 cm}
\caption{ Accuracy vs. $N$.} \label{fig:acc_Impact_N}
%\end{varwidth}
\end{figure*}
\begin{figure*}[htp!]
\vspace{-1.5 cm}
\centering
%\begin{varwidth}{0.45\linewidth}
\subfigure[Pendigits.] {\label{fig:acc_Pen_L}\includegraphics[trim = 1.5cm 6.5cm
2.1cm 3cm, clip, width=0.25\textwidth]{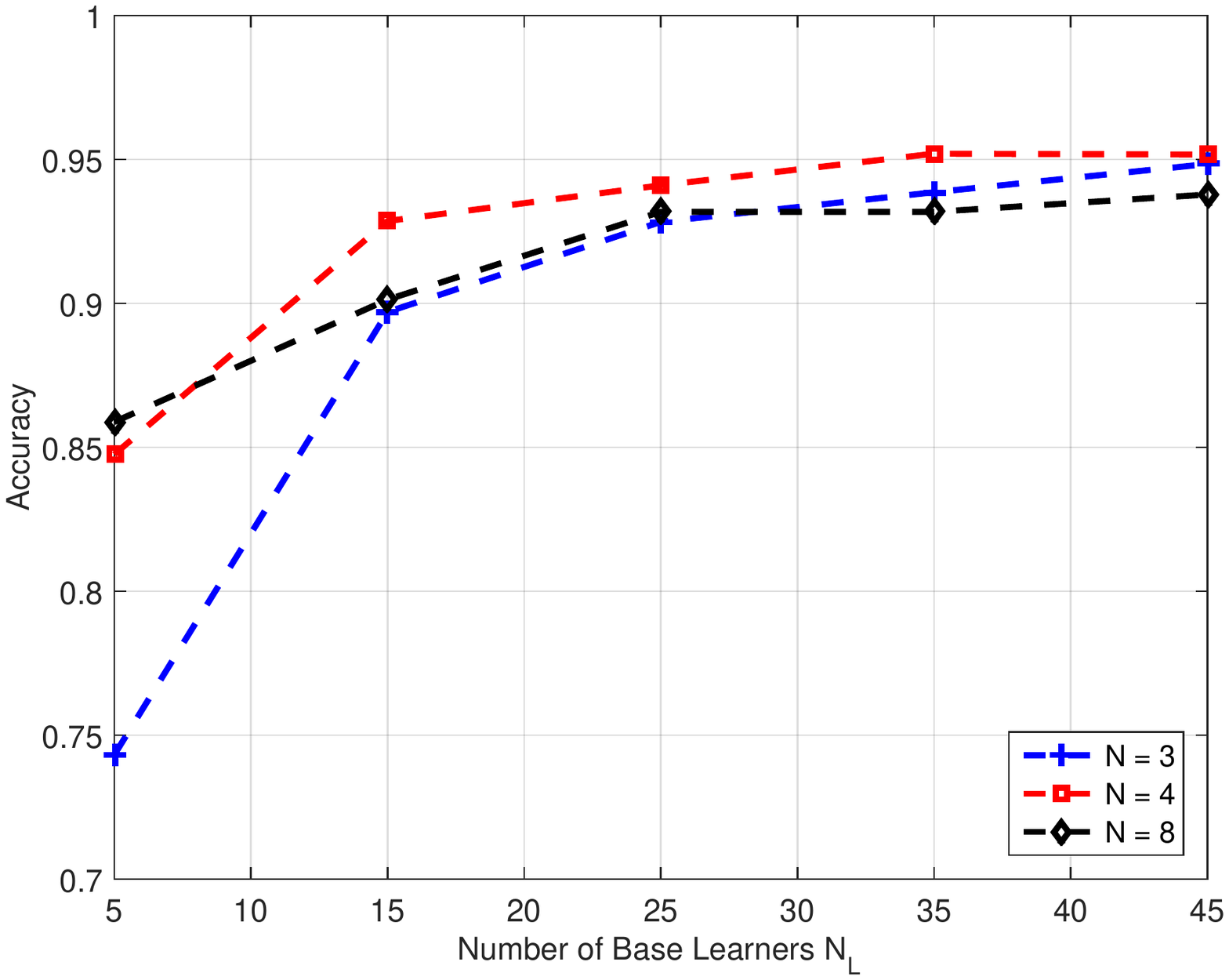}}%\vspace{-0.5 cm}
\subfigure[Letters.] {\label{fig:acc_Letter_L}\includegraphics[trim = 1.5cm 6.5cm
2.1cm 3cm, clip, width=0.25\textwidth]{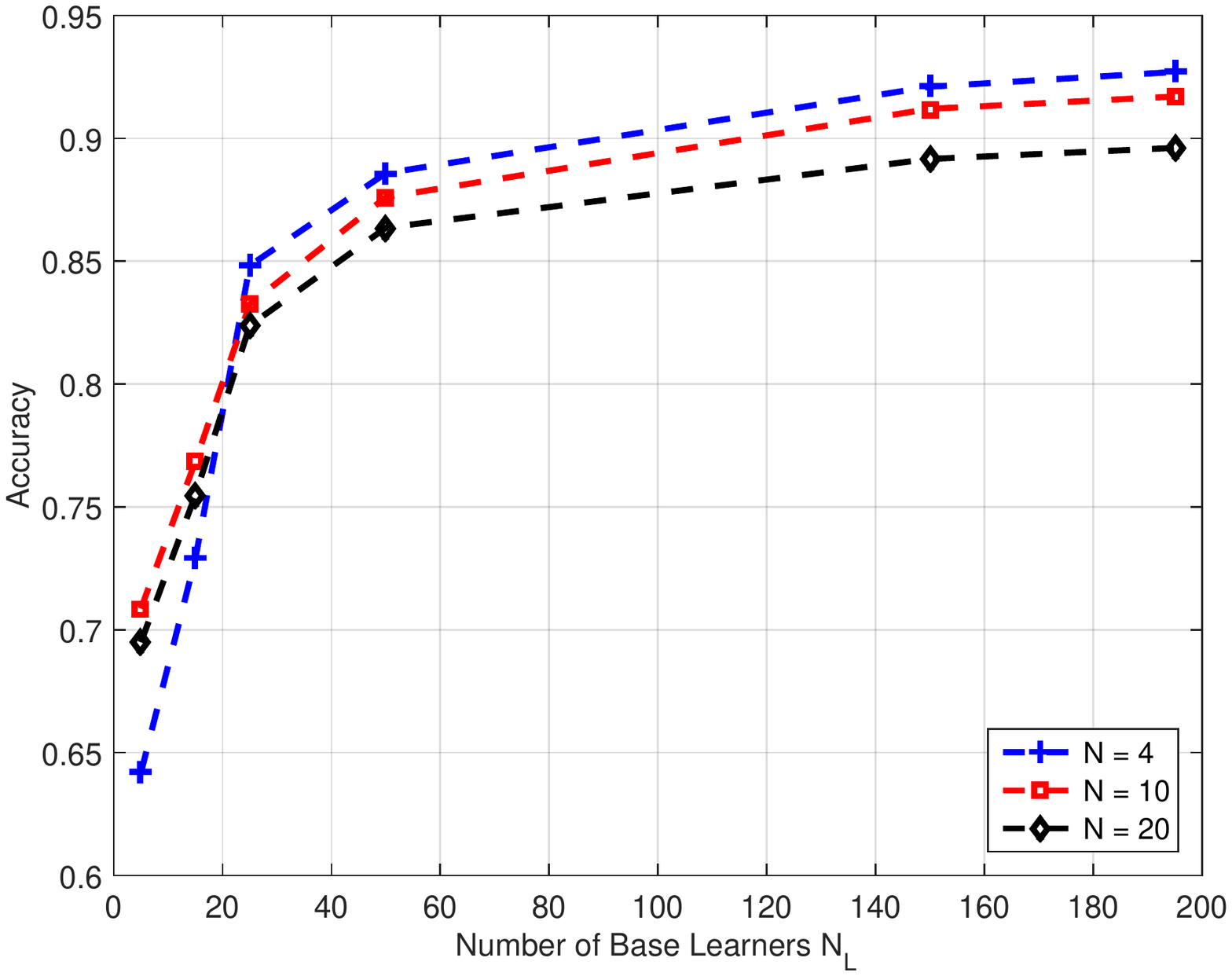}}%\vspace{-0.5 cm}
 \subfigure[ Sector.]{\label{fig:acc_Sector_L}\includegraphics[trim = 1.5cm 6.5cm
2.1cm 3cm, clip, width=0.25\textwidth]{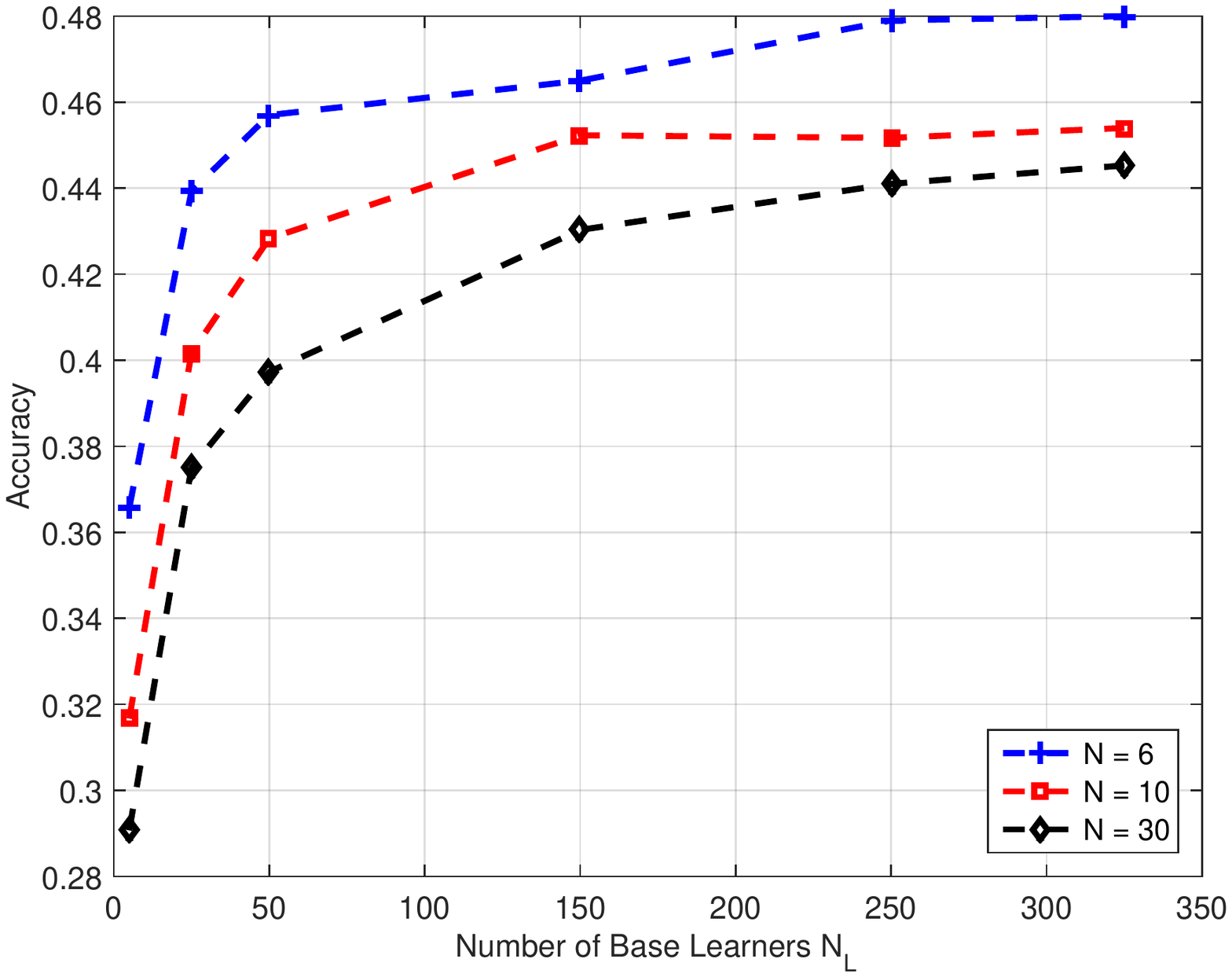}}%\vspace{-0.5 cm}
 \subfigure[ Aloi.]{\label{fig:acc_Aloi_L}\includegraphics[trim = 1.5cm 6.5cm
2.1cm 3cm, clip, width=0.25\textwidth]{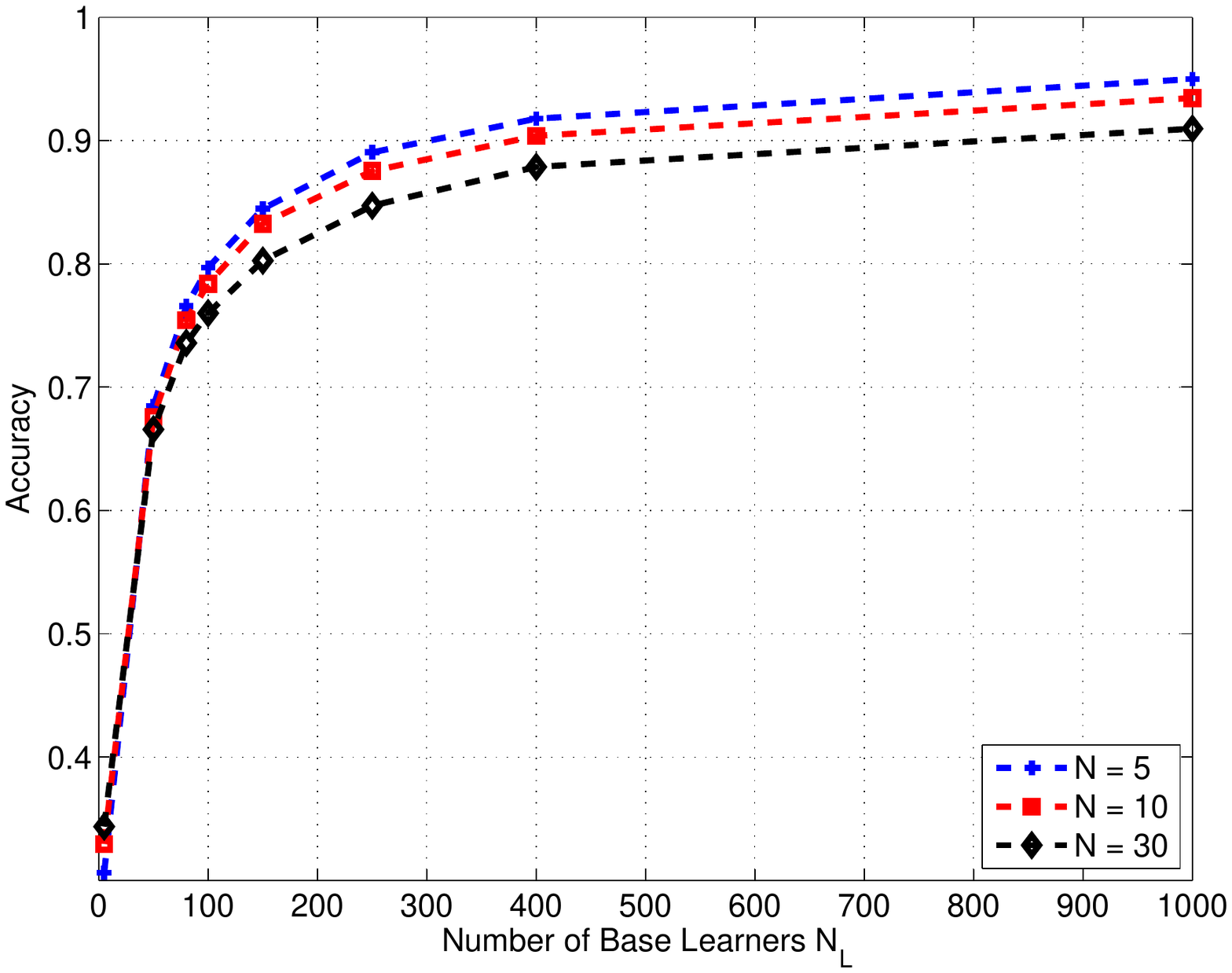}}%\vspace{-0.05 cm}
\caption{ Accuracy vs. $N_L$.} \label{fig:acc_Impact_L}
%\end{varwidth}
\end{figure*}

\subsection{Comparison to State-of-the-art ECOC Strategies.} %\vspace{-1mm}
We compare our proposed $N$-ary ECOC scheme to other state-of-the-art ECOC schemes with different base classifiers including decision tree (DT)  \cite{cart84-2} and support vector machine (SVM) \cite{Chang:2011:LLS:1961189.1961199}. \footnote{Note that coding design is independent from base learners. It is fair to fix the base learners for ECOC coding comparison.} The two binary classifiers can be easily extended to a multi-class setting. In particular, we use the multi-class SVM (M-SVM) \cite{Crammer:2002:AIM:944790.944813} implemented with the MSVMpack \cite{MSVMpack}. In addition to the multi-class extension of the two classifiers, we also compare $N$-ary ECOC to OVO, OVA, random ECOC, ECOCONE and DECOC with the two binary classifiers. For random ECOC and $N$-ary ECOC, we report the best  results with  $N_L \leq N_C(N_C-1)/2$, which is sufficient for conventional random ECOC to reach optimal performance \cite{Allwein:2001:RMB:944733.944737,4668347}. But for Aloi dataset with 1000 classes, we only report the results for all the ECOC based methods within $N_L=1000$ due to its large class size.
%\vspace{-0.1 mm}

\subsubsection{Comparison to state-of-the-art ECOC  with SVM Classifiers}
The classification accuracy of different ECOC coding schemes as well as proposed $N$-ary ECOC with SVM classifiers are presented in
 Table \ref{tab:summary_datasets_SVM}.  We
 observe that OVO  has  the best and most stable performance on most datasets of all the encoding schemes except for $N$-ary ECOC. This is because  all the information between any two classes is used during classification and the OVO coding strategy has no redundancy among different base classifiers.
However, it sacrifices  efficiency for
better performance. It is very expensive for both training and testing when there are many classes in the datasets such as the Auslan, Sector and Aloi. Especially, for Aloi with 1000 classes, it is often not viable to calculate the entire
OVO classifications in the real-world application as it would require 499 500
 base
learners in the pool of possible combinations for training and
testing.
The performance of OVA is unstable. For the datasets News20 and Sector, OVA  even significantly outperforms OVO. However, the performances of OVA on the datasets Vowel, Letters, and Glass are much worse than other encoding schemes.  Note that ECOCONE is initialized with OVA. Its
performance largely depends on the performance of the initial coding
matrix. When OVA performs poorly, ECOCONE  also performs poorly.  Another
problem-dependent coding approach DECOC sets the fixed length of
ECOC codes to $N_C-1$. Although ECOCONE and DECOC are problem-dependent coding strategies, their performance is not satisfactory in general. We observe that M-SVM achieves better results than ECOC because it considers relationship among classes. However, the training complexity of M-SVM is very high. In contrast to M-SVM,  ECOC can be parallelized due to independencies of base tasks. $N$-ary ECOC combines the advantages of both M-SVM and ECOC to achieve better performance.
%\begin{sidewaystable}
%\renewcommand*\arraystretch{1}
%\vspace{0.2 cm}
%\centering
\begin{table*}
\small
\caption{Classification accuracy and standard deviation obtained by SVM classifiers for UCI datasets.}\
%\begin{adjustbox}{max width=\textwidth}
\begin{tabular}{c| c c c c c c c}
  \hline
    Dataset   &OVO              &OVA                   &ECOC                 & DECOC              & ECOCONE            & M-SVM           & Nary-ECOC  \\\hline
    Pendigits  & \textbf{93.71 $\pm$ 2.03} &$81.75 \pm 1.07$   & $87.64 \pm 1.08$       & $72.21 \pm 0.88$  & $88.84 \pm 1.57$  &$89.85 \pm 1.65$  & $93.22 \pm 1.71$\\
    Vowel      & \textbf{48.67 $\pm$ 2.63} &$30.30 \pm 1.42$   & $34.28 \pm 1.55$       & $32.42 \pm 1.49$  & $31.67 \pm 1.75$ & $41.67 \pm 1.42$  & $39.96 \pm 2.07$\\
    News20     & $68.36 \pm 1.70$ &$72.65 \pm 2.24$   & $70.61 \pm 1.67$       & $66.84 \pm 1.90$   & $70.78 \pm 1.22$ & 70.78 $\pm$ 0.85& \textbf{72.79 $\pm$ 1.08} \\
    Letters    & $81.85 \pm 1.26$ &$66.93 \pm 1.37$   & $77.78 \pm 1.26$      & $68.58 \pm 2.35$   & $70.56 \pm 1.58$ & 80.16 $\pm$ 2.06  & \textbf{82.27 $\pm$ 1.09}\\
    Auslan     & $89.94 \pm 2.23$ &$41.12 \pm 1.28$   & $83.15 \pm 0.84$      & $50.86 \pm 2.20$   & $46.24 \pm 1.60$ & $90.06 \pm 1.58$ & \textbf{91.57 $\pm$ 0.96}\\
    Sector     & $85.06 \pm 1.54$ &$89.06 \pm 1.35$    & $88.01 \pm 1.47$     & $88.47 \pm 1.05$   & $89.89 \pm 0.88$  & $88.75 \pm 1.55$  &\textbf{ 91.05 $\pm$ 1.32}\\
    Aloi       & 91.49 $\pm$ 1.68 & $84.26 \pm 2.53$   & $85.48 \pm 1.17$ & $82.47 \pm 1.05$   & $85.09 \pm 0.88$  & $86.69 \pm 1.02$  & \textbf{92.77 $\pm$ 1.86}\\
    Glass      & \textbf{61.84 $\pm$ 2.24} &$55.00 \pm 2.23$   & $56.00 \pm 0.89$    & $52.21 \pm 1.21$   & $56.84 \pm 1.12$  & $58.84 \pm 2.42$  & $56.84 \pm 1.56$\\
    Satimage   & $85.69 \pm 2.57$ &$83.11 \pm 1.86$    & $81.19 \pm 1.62$    & $82.71 \pm 1.44$   & $82.56 \pm 1.61$  & $83.28 \pm 0.36$  &\textbf{ 86.50 $\pm$ 0.93}\\
    Usps       & $94.30 \pm 1.14$ &$92.37 \pm 1.65$    & $90.76 \pm 1.45$    & $83.11 \pm 1.72$   & $78.28 \pm 1.67$  & $92.16 \pm 1.27$  & \textbf{96.15 $\pm$ 2.47}\\
    Segment    & $92.30 \pm 1.46$ &$92.00\pm 1.89$    & $87.80 \pm 1.61$     & $90.78 \pm 1.78$  & $78.28 \pm 1.23$  & $89.96 \pm 1.49$  & \textbf{93.60 $\pm$ 1.53}\\
    Mean Rank  & 2.8              & 4.9               &5.1                   &5.6                 &5.0                 & 3.0                 &\textbf{1.6}
\end{tabular}\label{tab:summary_datasets_SVM}
%\end{sidewaystable}
%\begin{sidewaystable}
%\renewcommand*\arraystretch{1}
%\vspace{0.2 cm}
\centering \caption{Classification accuracy and standard deviation obtained by CART classifiers for UCI datasets.}
\small
\begin{tabular}{c| c c c c c c c}
  \hline
    Dataset    &OVO               &OVA                   &ECOC                 & DECOC             & ECOCONE            & M-CART           & Nary-ECOC  \\\hline
    Pendigits  & $93.84 \pm 2.33$ &$78.12 \pm 1.24$   & $83.54 \pm 1.30$       & $81.45 \pm 1.38$  & $80.53 \pm 1.12$  &$87.64 \pm 1.12$  & \textbf{95.84 $\pm$ 1.08}\\
    Vowel      & $44.45 \pm 1.74$ &$33.57 \pm 2.31$   & $43.65 \pm 2.35$       & $35.73 \pm 1.16$  & $36.15 \pm 2.25$ & $45.45 \pm 2.25$  & \textbf{48.50 $\pm$ 1.20}\\
    News20     & $50.60 \pm 1.17$ &$45.23 \pm 1.15$   & $51.29 \pm 1.26$      & $44.25 \pm 2.29$   & $50.28 \pm 1.37$ & $50.83 \pm 1.37$  & \textbf{53.71 $\pm$ 1.70}\\
    Letters    & $81.56 \pm 1.30$ &$74.69 \pm 1.50$   & $89.75 \pm 1.55$      & $78.56 \pm 1.05$   & $77.10 \pm 1.26$ & $77.35 \pm 1.26$  & \textbf{92.15 $\pm$ 1.92}\\
    Auslan    & $79.84 \pm 2.23$ &$72.86  \pm 2.04$   & $83.15 \pm 2.84$     & $75.30 \pm 1.15$   & $75.28 \pm 2.36$   & $78.89 \pm 1.18$ &\textbf{85.17 $\pm$ 1.26}\\
    Sector     & $39.49 \pm 1.33$ &$41.89 \pm 1.26$   & $43.60 \pm 1.17$     & $44.47 \pm 2.35$   & $44.29 \pm 1.18$  & $45.89 \pm 2.15$  &\textbf{ 47.05 $\pm$ 1.27}\\
    Aloi     & $89.26 \pm 1.49$ &$72.10 \pm 2.60$   & $79.41 \pm 1.08$     & $75.33 \pm 2.13$   & $72.78 \pm 1.80$  & $73.00 \pm 2.07$  &\textbf{ 95.13 $\pm$ 1.89}\\
    Glass      & $52.84 \pm 1.15$ &$50.12 \pm 1.24$   & $54.65 \pm 1.35$    & $52.21 \pm 2.38$   & $53.02 \pm 2.12$   & \textbf{64.00 $\pm$ 2.12}  & 56.00 $\pm$ 1.14\\
    Satimage   & $85.70 \pm 1.27$ &$84.15 \pm 1.08$    & $85.86 \pm 2.75$    & $80.37 \pm 1.16$   & $84.28 \pm 2.36$  & $83.47 \pm 2.36$  & \textbf{86.47 $\pm$ 2.23}\\
    Usps       & $90.94 \pm 2.16$ &$80.89 \pm 1.47$    & $91.95 \pm 1.57$    & $80.25 \pm 1.19$   & $80.45 \pm 2.36$  & $83.54 \pm 1.16$  & \textbf{92.77 $\pm$ 1.15}\\
    Segment    & $93.68 \pm 1.25$ &$86.45 \pm 2.22$    & $96.44 \pm 2.52$     & $80.57 \pm 2.16$  & $88.78 \pm 1.36$  & $92.70 \pm 1.34$  & \textbf{97.10 $\pm$ 1.28}\\
    Mean Rank  & 3.6              & 6.3               &2.8                   &5.7                &5.1                 & 3.4                 &\textbf{1.1}
\end{tabular}\label{tab:summary_datasets_DT}
%\end{sidewaystable}
\end{table*}
%\begin{figure*}[htp!]
%%\vspace{-3.5 cm}
%\centering
%\subfigure[ Binary Code.] {\label{fig:confu_bi}\includegraphics[trim = 4.5cm 7.5cm
%5.1cm 5.5cm,clip, width=0.3\textwidth]{ova_confusion.pdf}}%\vspace{-1.5cm}
%\subfigure[ Ternary Code.] {\label{fig:confu_ter}\includegraphics[trim = 4.5cm 7.5cm
%5.1cm 5.5cm,clip, width=0.3\textwidth]{ovo_confusion.pdf}}%\vspace{-0.1 cm}
% \subfigure[ N-ary Code.]{\label{fig:confu_nary}\includegraphics[trim = 4.5cm 7.5cm
%5.1cm 5.5cm, clip, width=0.3\textwidth]{Nary_confusion_.pdf}}%\vspace{-0.4 cm}
%\caption{Confusion matrix on Pendigits: In confusion
%matrix, the entry in the ith row and jth column is the percentage
%of images from class i that are misidentified as class j. Average
%classification rates for individual classes are listed along the diagonal.} \label{fig:confusion}
%\end{figure*}

\begin{figure}[htp!]
%\vspace{-3.5 cm}
\centering
\subfigure[ Binary Code.] {\label{fig:confu_bi}\includegraphics[trim = 3cm 18cm
2.5cm 2.5cm, clip, width=0.525\textwidth]{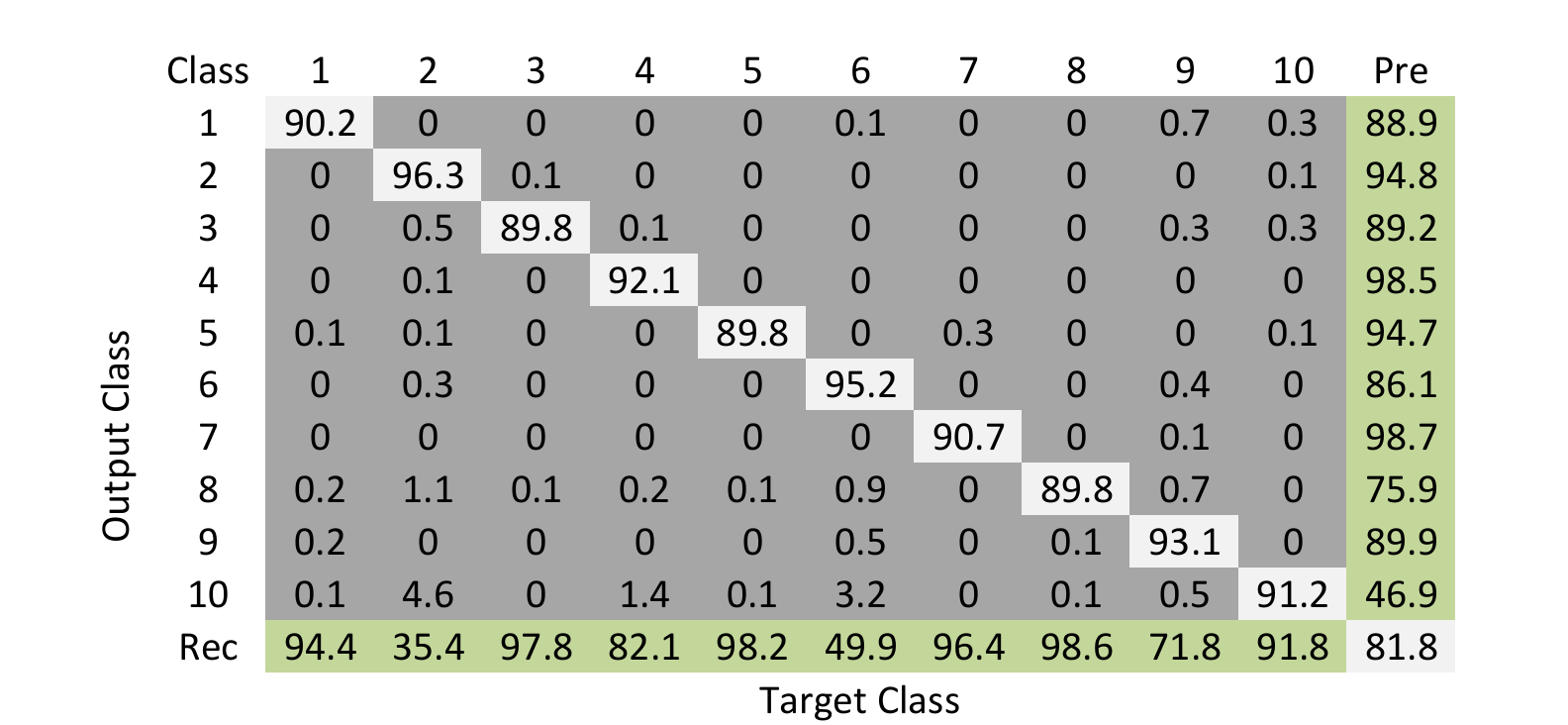}}\\%\vspace{-1.5cm}
\subfigure[ Ternary Code.] {\label{fig:confu_ter}\includegraphics[trim = 3cm 18cm
2.5cm 2.5cm, clip, width=0.53\textwidth]{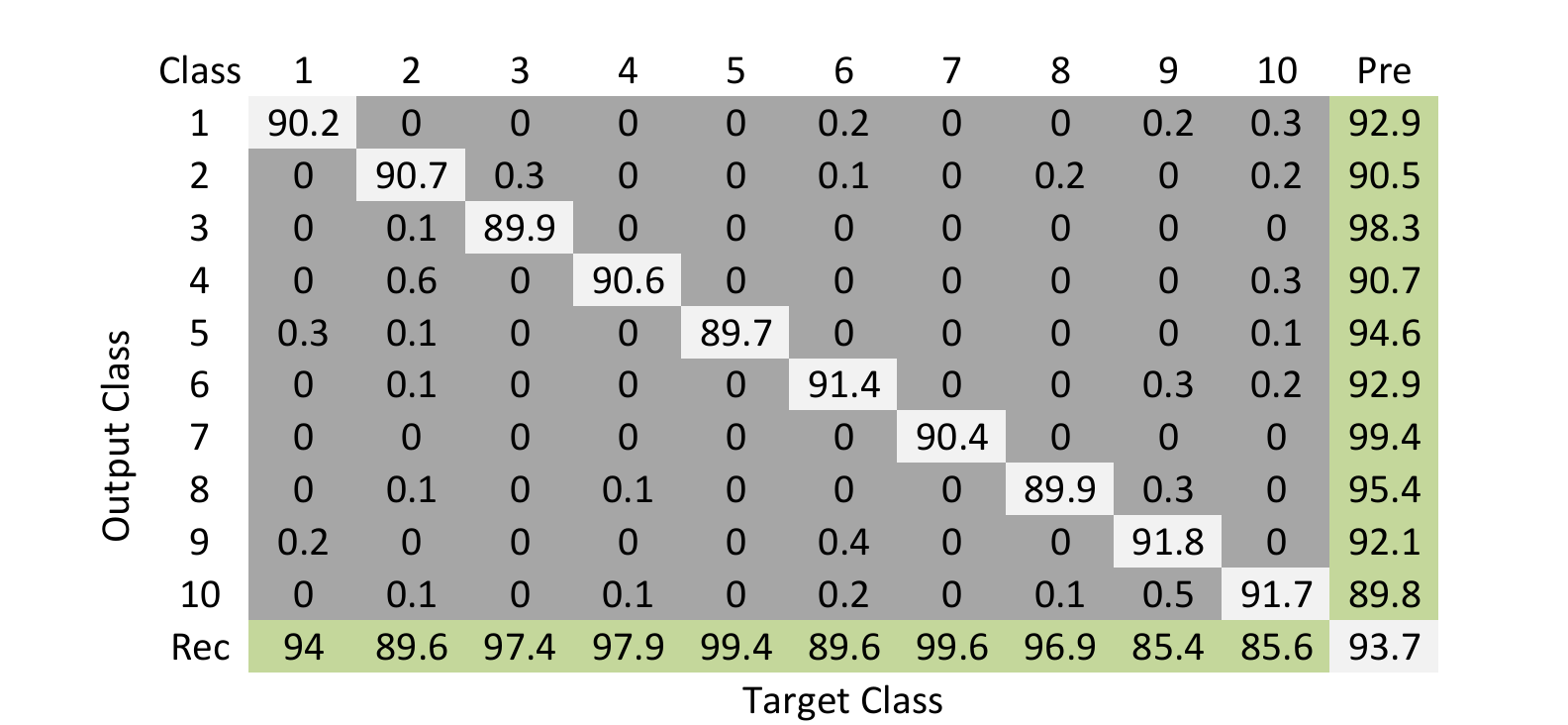}}\\%\vspace{-0.1 cm}
 \subfigure[ N-ary Code.]{\label{fig:confu_nary}\includegraphics[trim = 3.1cm 18cm
2.5cm 2.5cm, clip, width=0.525\textwidth]{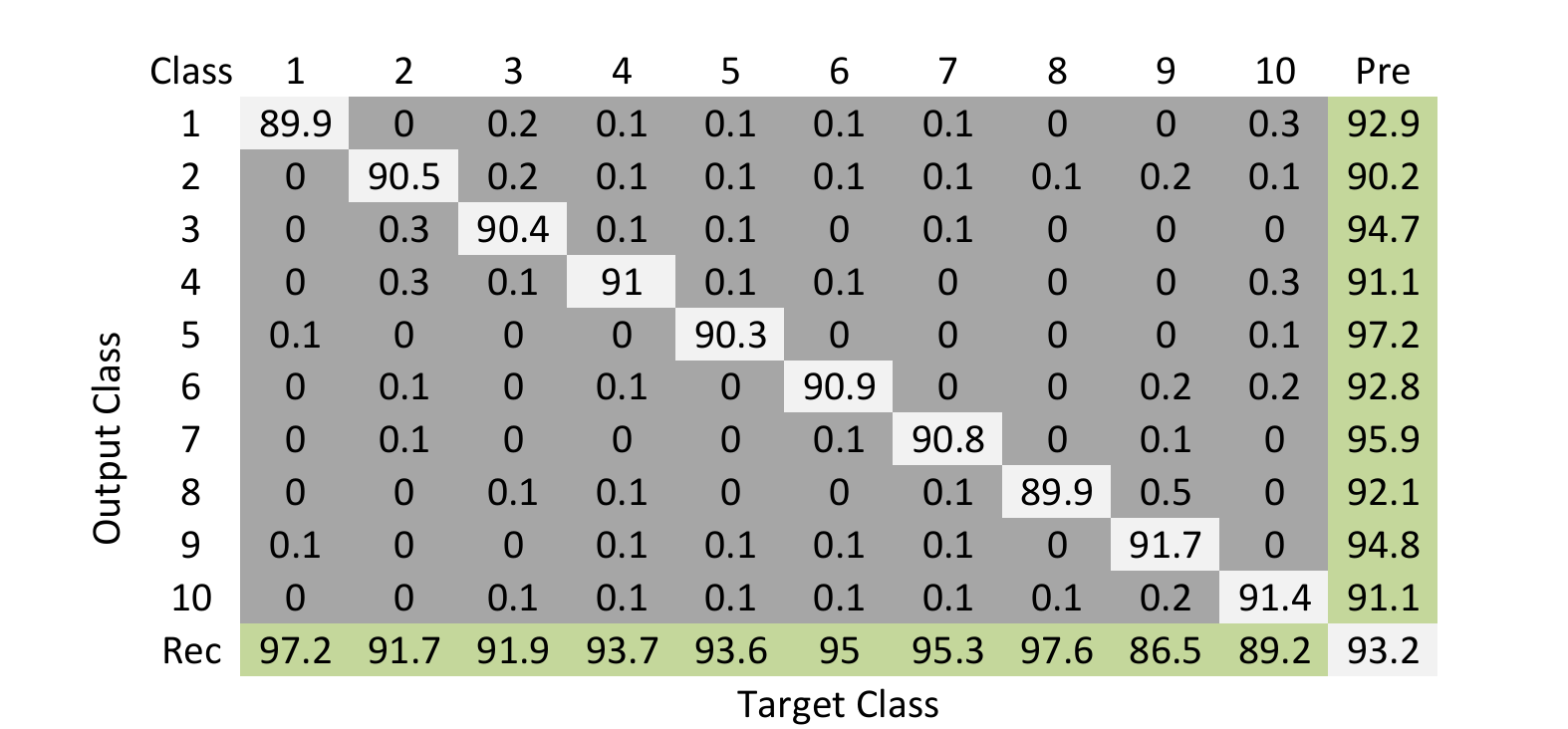}}%\vspace{-0.4 cm}
\caption{Confusion matrix on Pendigits: In confusion
matrix, the entry in the ith row and jth column is the percentage
of images from class i that are misidentified as class j. Average
classification rates for individual classes are listed along the diagonal. The last column and last row are  precision (Pre) and recall (Rec) respectively.  } \label{fig:confusion}
\end{figure}

\subsubsection{Comparison to State-of-the-art ECOC  with Decision Tree Classifiers}
Next, we compare $N$-ary ECOC with other state-of-the-art coding schemes using binary decision tree classifiers CART  \cite{cart84-2} as well as its multi-class extension M-CART. We implement it with the CART toolbox with a default setting and the results are reported in Table \ref{tab:summary_datasets_DT}. We observe that binary decision tree classifiers with traditional ECOC strategies are worse than the direct multi-class extension of the decision tree. The decision tree classifiers show  better performances than SVM on the Pendigits, Vowel, and Letters datasets. However, it shows very poor performances on high dimensional datasets such as News20 and Sector. This is due to the fact that high-dimensional features often lead to complex tree structure  construction. Nevertheless, $N$-ary ECOC still can significantly improve the performance on either traditional coding schemes with binary decision tree learner as well as the multi-class decision tree.

In summary, our proposed $N$-ary ECOC is superior to traditional ECOC encoding schemes and direct multi-class algorithms on  most tasks, and provides a flexible strategy to decompose many classes into many smaller multi-class problems, each of which can be independently solved by  either M-SVM or M-CART in parallelization.

\subsubsection{Discussion on Many class Situation}
From the experiments results, we observe that the $N$-ary ECOC shows significant  improvement on the Aloi dataset with 1000 classes over other existing coding schemes as well as  direct multi-class classification algorithms, especially  decision tree classifiers. For the binary or ternary ECOC codes, it is highly possible to assign the same codes to different classes. From the experimental results, we observe the minimum distance $\rho$ for binary and ternary coding are small or even tends to be 0.  In other words, the existing coding cannot help the classification algorithms to differentiate some classes.  In contrast, $N$-ary with $N_L = 1000$ and $N=5$,  the minimum distance $\rho$ is 741. Thus, it creates codes with larger margins for different classes, whichexplains the superior    On the other hand, the direct multi-class algorithms cannot work well when the class size is large. Furthermore, the computation cost for direct multi-class algorithms is in $O(N_C^3)$. When the class size $N_C$ is large, the algorithms  are expensive to train. On the contrary, random ECOC codes can be easily parallelized due to the independency among the subproblems.

\subsubsection{Discussion on Performance on Each Individual Class}
To understand the performances of different codes for each individual class,  we show the confusion matrix on the Pendigits dataset in Figure \ref{fig:confusion}.
First, we observe that binary code (i.e., OVA) has very poor performances on some classes in terms of recall or precision.
For example, recall on class 2, 6 and precision on class 10 are below 50\%. It can be explained by that as illustrated in Figure \ref{fig:intro:binary}, binary codes may lead  to nonseparable cases.
 Nevertheless, it achieves best classification results on the class 2, 4, 6 and class 9.
Compared to the binary code, ternary code (i.e., OVO) largely reduces the bias and improve precision and recall scores on most classes. What is more interesting, when the ternary code and $N$-ary code achieves comparable overall performances, $N$-ary achieves smaller maximal errors.
It may be benefited from simpler subtasks created by $N$-ary coding scheme, as shown in Figure \ref{fig:intro:nary}.

%\begin{table}[h]
%\renewcommand*\arraystretch{1}
%\centering \caption{\label{tab:summary_knn} Summary of KNN performances on the datasets used in the experiments.}
%{\small
%\begin{tabular}{c|c}
%  \hline
%    Dataset    & knn\\\hline
%    Pendigits  & $97.74 \pm 1.22$\\
%    Vowel      & $56.28 \pm 1.81$\\
%    News20     & $55.07 \pm 1.16$\\
%    Letters    & $94.54 \pm 2.05$\\
%    Auslan     & $92.45 \pm 1.42$\\
%    Sector     & $85.20 \pm 2.03$\\
%    Aloi       & $93.79 \pm 1.67$\\
%    Glass      & $64.04 \pm 2.34$\\
%    Satimage   & $88.25 \pm 1.45$\\
%    Usps       & $88.25 \pm 2.56$\\
%    Segment    & $94.12 \pm 2.01$
%\end{tabular}
%%\vspace{-4mm}
%}
%\end{table}

\section{Conclusions}\label{sec:conclude}
In this paper, we investigate whether one can relax binary and ternary code design to $N$-ary code design to achieve better classification performance. In particular, we present an $N$-ary coding scheme that decomposes the original multi-class problem into simpler multi-class subproblems. The advantages of such a coding scheme are as follows: (i) the ability to construct more discriminative codes and (ii) the flexibility for the user to select the best $N$ for ECOC-based classification. We derive a base classifier independent generalization error bound for the $N$-ary ECOC classification problem. We show empirically that the optimal $N$ (based on classification performance) lies in $[3, 10]$ with some tradeoff in computational cost.  Experimental results on  benchmark multi-class datasets show that the proposed coding scheme achieves superior prediction performance over the state-of-the-art coding methods. In the future, we will investigate a  more efficient realization of  $N$-ary coding scheme to improve the prediction speed.

\section{Acknowledgements}
This work  is done when Dr Joey Tianyi Zhou were at Nanyang Technological University supported by the Research Scholarship.
Dr. Ivor W. Tsang is grateful for the support from the ARC Future Fellowship FT130100746 and ARC grant LP150100671.
Dr. Shen-Shyang Ho acknowledges the support from MOE Tier 1 Grants RG41/12.
Dr. Klaus-Robert~M$\ddot{\textrm{u}}$ller gratefully acknowledges partial financial support by DFG, BMBF and BK21 from NRF (Korea).

\bibliography{TIT_arxiv}
\bibliographystyle{plain}
\begin{IEEEbiography}{Joey Tianyi Zhou}
 is a scientist  with the Computing Science Department, Institute of High Performance Computing (IHPC), Singapore.  Prior to join IHPC, he was a research fellow in Nanyang Technological University (NTU). He received the Ph.D. degree in computer
science from NTU, Singapore, in 2015.
\end{IEEEbiography}

\begin{IEEEbiography}{Ivor W. Tsang}
 is an Australian Future
Fellow and an Associate Professor with the Centre
for Quantum Computation and Intelligent Systems,
University of Technology at Sydney, Ultimo, NSW,
Australia.  He received the Ph.D. degree in computer
science from the Hong Kong University of
Science and Technology, Hong Kong, in 2007.
He was the Deputy Director of the Centre for
Computational Intelligence, Nanyang Technological
University, Singapore.  He has authored more than 100 research
papers in refereed international journals and conference
proceedings, including JMLR, TPAMI, TNN/TNNLS, NIPS, ICML,
UAI, AISTATS, SIGKDD, IJCAI, AAAI, ACL, ICCV, CVPR, and ICDM.
Dr. Tsang was a recipient of the 2008 Natural Science Award (Class II) from
the Ministry of Education, China, in 2009, which recognized his contributions
to kernel methods. He was also a recipient of the prestigious Australian
Research Council Future Fellowship for his research regarding Machine
Learning on Big Data in 2013. In addition, he received the prestigious IEEE
TRANSACTIONS ON NEURAL NETWORKS Outstanding 2004 Paper Award in
2006, the 2014 IEEE TRANSACTIONS ON MULTIMEDIA Prized Paper Award,
and a number of Best Paper Awards and Honors from reputable international
conferences, including the Best Student Paper Award at CVPR 2010, the Best
Paper Award at ICTAI 2011, and the Best Poster Award Honorable Mention
at ACML 2012. He was also a recipient of the Microsoft Fellowship in 2005
and the ECCV 2012 Outstanding Reviewer Award
\end{IEEEbiography}

%\newpage

%\begin{IEEEbiographynophoto}{Shen-shyang Ho}

\begin{IEEEbiography}{Shen-shyang Ho}
 received the BS degree in mathematics
and computational science from the National
University of Singapore in 1999, and the MS
and PhD degrees in computer science from George
Mason University in 2003 and 2007, respectively.
From August 2007 to May 2010, he was a NASA
Postdoctoral Program (NPP) fellow and then a postdoctoral
scholar at the California Institute of Technology
affiliated to the Jet Propulsion Laboratory
(JPL). From June 2010 to December 2012, he was
a researcher working on projects funded by NASA
at the University of Maryland Institute for Advanced Computer Studies
(UMIACS). Currently, he is a tenure-track assistant professor in the School of
Computer Engineering at the Nanyang Technological University. His research
interests include data mining, machine learning, pattern recognition in spatiotemporal/data
streaming settings, and privacy issues in data mining. He is
also currently involved in industrial research projects funded by BMW and
Rolls-Royces. He has given tutorials at AAAI, IJCNN, and ECML.
\end{IEEEbiography}

\begin{IEEEbiography}{Klaus-Robert~M$\ddot{\textrm{u}}$ller}
received the Diploma degree
in mathematical physics and the Ph.D. degree in computer
science from Technische Universit$\ddot{\textrm{a}}$t Karlsruhe,
Karlsruhe, Germany, in 1989 and 1992, respectively. \\
He has been a Professor of computer science with
Technische Universit$\ddot{\textrm{a}}$t Berlin, Berlin, Germany,
since 2006, and he is directing the Bernstein Focus
on Neurotechnology Berlin, Berlin. Since 2012,
he has been a Distinguished Professor with Korea
University, Seoul, Korea, within the WCU Program.
After a Post-Doctoral at GMD-FIRST, Berlin, he
was a Research Fellow with the University of Tokyo, Tokyo, Japan, from 1994
to 1995. In 1995, he built up the Intelligent Data Analysis Group, GMD-FIRST
(later Fraunhofer FIRST) and directed it until 2008. From 1999 to 2006, he was
a Professor with the University of Potsdam, Potsdam, Germany. His current
research interests include intelligent data analysis, machine learning, signal
processing, and brain computer interfaces.\\
Dr. Müller received the Olympus Prize by the German Pattern Recognition
Society, DAGM, in 1999, and the SEL Alcatel Communication Award in 2006.
In 2012, he was elected to be a member of the German National Academy of
Sciences Leopoldina.
\end{IEEEbiography}

\end{document}